\PassOptionsToPackage{english}{babel}
\documentclass{new_tlp}

\usepackage[english]{babel}
\usepackage{comment}
\usepackage[utf8]{inputenc}
\usepackage{times}
\usepackage{xspace}
\usepackage{latexsym}
\usepackage{verbatim}
\usepackage{amsmath,amssymb}
\usepackage{upgreek}
\usepackage{dsfont}
\usepackage{framed}
\usepackage{amsfonts}
\usepackage{bm}
\usepackage[ruled,vlined]{algorithm2e}
\usepackage{pifont}
\newcommand{\xmark}{\ding{55}}
\usepackage{color}
\definecolor{gray75}{gray}{0.75}
\definecolor{grein}{rgb}{0.058, 0.298, 0.408}
\definecolor{brein}{rgb}{0.55, 0.21, 0.000}
\definecolor{lightblue}{rgb}{0,0.5,0.75}
\usepackage{longtable}

\usepackage{pgf,tikz,pgfplots}
\pgfplotsset{compat=1.14}  

\usepackage{mathrsfs}

\usetikzlibrary{arrows}
\usepackage{enumerate}
\usepackage{float}

\newtheorem{theorem}{Theorem}

\newtheorem{proposition}[theorem]{Proposition}

\newtheorem{corollary}[theorem]{Corollary}

\newtheorem{definition}[theorem]{Definition}

\newtheorem{example}[theorem]{Example}

\renewenvironment{proof} {\noindent\emph{Proof:}} {\hfill $\square$\vspace{0.2cm}}

\newcommand{\fun}[1]{\ensuremath{\mbox{\it #1}}}

\newcommand{\var}{\ensuremath{\mathit{var}}\xspace}

\newcommand{\term}{\ensuremath{\mathit{term}}\xspace}
\newcommand{\trg}{\ensuremath{\mathit{triggers}}\xspace}

\newcommand{\fr}{\ensuremath{\mathit{fr}}\xspace}

\newcommand{\cnst}{\ensuremath{\mathit{cnst}}\xspace}

\newcommand{\OB}{\ensuremath{\fun{$\mathbf{O}$}}}
\newcommand{\SO}{\ensuremath{\fun{$\mathbf{SO}$}}}
\newcommand{\RE}{\ensuremath{\fun{$\mathbf{R}$}}}

\newcommand{\bfE}{\ensuremath{\fun{$\mathbf{bf}\textbf{-}\mathbf{E}$}}}
\newcommand{\bfR}{\ensuremath{\fun{$\mathbf{bf}\textbf{-}\mathbf{R}$}}}
\newcommand{\bfSO}{\ensuremath{\fun{$\mathbf{bf}\textbf{-}\mathbf{SO}$}}}
\newcommand{\bfO}{\ensuremath{\fun{$\mathbf{bf}\textbf{-}\mathbf{O}$}}}
\newcommand{\bfX}{\ensuremath{\fun{$\mathbf{bf}\textbf{-}$}}X}

\newcommand{\EQ}{\ensuremath{\fun{$\mathbf{E}$}}}

\newcommand{\D}{\ensuremath{\fun{$\mathcal{D}$}}\xspace}
\newcommand{\DD}{\ensuremath{\fun{$^{\mathcal{D}}$}}}
\newcommand{\dD}{\ensuremath{\fun{$_{\mathcal D}$}}}

\renewcommand{\t}{\ensuremath{\fun{${t}$}}}

\renewcommand{\tt}[1]{\ensuremath{\fun{$_{{t}_{#1}}$}}}

\newcommand{\spp}{\ensuremath{\fun{$\mathit{support}$}}}
\newcommand{\op}{\ensuremath{\mathit{output}}}

\definecolor{gren}{rgb}{0.158, 0.488, 0.408}
\definecolor{bren}{rgb}{0.45, 0.21, 0.000}
\definecolor{ggray}{rgb}{0.65, 0.65, 0.65}

\newcommand{\correction}[1]{#1} 
\newcommand{\shmeiwsh}[1]{#1} 
\newcommand{\corvtwo}[1]{#1}

\title{Characterizing Boundedness in Chase Variants}

\author[Stathis Delivorias, Michel Lecl\`ere, Marie-Laure Mugnier, Federico Ulliana]{Stathis Delivorias, Michel Leclère, Marie-Laure Mugnier, Federico Ulliana\\ University of Montpellier, LIRMM,  CNRS, Inria,
Montpellier, France}

\submitted{8 April 2019}
\revised{16 October 2019}
\accepted{ --- } 

\makeindex
    
\begin{document}
\maketitle

\begin{abstract}
Existential rules are a positive fragment of first-order logic that generalizes function-free Horn rules by allowing existentially quantified variables in rule heads.
This family of languages has recently attracted significant interest in the context of ontology-mediated query answering. 
Forward chaining, also known as the chase, is a fundamental tool for computing universal models of  knowledge bases, which consist of existential rules and facts. 
Several chase variants have been defined, which differ on the way they handle redundancies. 
A set of existential rules is bounded if it ensures the existence of a bound on the depth of the chase, independently from any set of facts.
Deciding if a set of rules is bounded is an undecidable problem for all chase variants. 
Nevertheless, when computing universal models, knowing that a set of rules is bounded for some chase variant does not help much in practice if the bound remains unknown or even very large.  Hence, we investigate the decidability of the k-boundedness problem, which asks whether the depth of the chase for a given set of rules is bounded by an integer k. 
We identify a general property which, when satisfied by a chase variant, leads to the decidability of k-boundedness. 
We then show that the main chase variants satisfy this property, namely the oblivious, semi-oblivious (aka Skolem), and restricted chase, as well as their breadth-first versions. 
\textit{This paper is under consideration for publication in Theory and Practice of Logic Programming. }
\end{abstract}
%

\correction{\section{Introduction}}
\sloppy
Existential rules (see  \cite{cgk08,blms09,cgl09} for the first papers and \cite{GottlobOPS12,MugnierT14} for introductory courses) are a positive fragment of first-order logic that generalizes function-free Horn rules, \corvtwo{such as} the deductive database query language Datalog, and knowledge representation formalisms such as Horn description logics (see e.g. \cite{dl-lite05,BBL-EL,krh07}).  
More specifically, existential rules are of the form \emph{body $\rightarrow$ head}, where 
\emph{body} and \emph{head} are conjunctions of
atoms (without functions), and variables that occur only in the head are existentially quantified. 
These existentially quantified variables allow one to assert
the existence of unknown individuals, a key feature \corvtwo{for reasoning} on incomplete data with the open-domain assumption.

Existential rules have the same logical form as \corvtwo{the} general database constraints known as tuple-generating dependencies, 
which have long been investigated in database theory \cite{AbiteboulHV}.  
 Reborn under the names of existential rules,  Datalog$^+$ or Datalog$^\exists$, they have attracted significant interest in the last years as ontological languages, especially for 
 ontology-mediated query answering and ontology-based data integration. 


\smallskip
In our setting, a knowledge base (KB) is composed of a set of existential rules, which typically encodes ontological knowledge, 
and a factbase, \corvtwo{a.k.a. instance,} which contains factual data.  
 A factbase is a set of atoms built from constants and variables, \corvtwo{the variables representing \mbox{unknown} individuals, also called (labeled) \emph{nulls} in databases}. The logical translation of a factbase is  an existentially closed conjunction of atoms. In this paper, we focus on the forward chaining process for reasoning on KBs, which consists of iteratively extending the factbase with new facts produced by rule applications, until we reach a fixpoint. 
In the forward chaining process, a rule of the form \emph{body $\rightarrow$ head} can be applied to a factbase $F$ whenever there is a homomorphism $h$ from \emph{body} to $F$. The factbase $F$ is then extended with new atoms obtained by first applying $h$ as a substitution to \emph{head} and then by 
\corvtwo{ renaming each existentially quantified variable (which has no image in $h$) with a fresh variable, i.e., that does not occur in $F$.} Hence, a rule application may produce new (unknown) individuals, \corvtwo{ i.e., nulls.} This is illustrated by the following example.

\begin{example}\label{ex-init}
Consider the KB $\mathcal K = ( F, \mathcal R)$ where 
$F = \{\textit{Human}(\textit{Alice})\}$ \corvtwo{is the factbase}
and 
$\mathcal R$ \corvtwo{is the ruleset containing the} 
single rule \mbox{$R=\forall x.~\textit{Human}(x) \rightarrow \exists y.~\textit{parentOf}(y,x) \land \textit{Human}(y)$}. 
Here, $R$ can be applied to $F$ because of  the homomorphism $\{x \mapsto \textit{Alice}\}$. This produces the atoms $\textit{parentOf}(y_0,\textit{Alice})$ and $\textit{Human}(y_0)$ containing a fresh variable  $y_0$ introduced by safely renaming $y$ in the rule. 
The factbase resulting from the rule application is therefore $\{\textit{Human}(\textit{Alice}), \textit{parentOf}(y_0,\textit{Alice}), \textit{Human}(y_0)\}$, which stands for the logical formula 
$\exists y_0. ~\textit{Human}(\textit{Alice}) \land \textit{parentOf}(y_0,\textit{Alice}) \land \textit{Human}(y_0)$. 
Note that $R$ can now be applied again by mapping $x$ to $y_0$ thereby creating a new individual $y_1$ and producing the atoms $\textit{parentOf}(y_1,y_0)$ and $\textit{Human}(y_1)$. 
In this example, the forward chaining does not terminate. 
\end{example}

Forward chaining with existential rules is also known as the \emph{chase} in databases and a considerable literature has been devoted to its analysis in the context of tuple-generating dependencies \cite{bv84,fkmp05,deutsch-nash-remmel08,marnette09,DBLP:journals/fuin/GrahneO18}. 
As illustrated by Example \ref{ex-init}, the chase may not terminate on a given KB. 
\corvtwo{However, a fundamental property of the chase is that it computes a \emph{universal model} of the  knowledge base, i.e., a model that homomorphically maps to any other model of the knowledge base \cite{deutsch-nash-remmel08}. }
This has a major implication in problems like answering \corvtwo{ontology-mediated queries,} \corvtwo{since a Boolean conjunctive query} (i.e., an existentially closed conjunction of atoms) is logically entailed by a KB if and only if it homomorphically maps to the result of the chase.

\sloppy
Several variants of the chase have been studied, mainly: the \emph{oblivious}  chase  \cite{cgk08}, the \emph{Skolem} chase  \cite{marnette09}, the \emph{semi-oblivious} chase \cite{marnette09}, the \emph{restricted} or standard chase \cite{fkmp05}, the \emph{core} chase \cite{deutsch-nash-remmel08} \mbox{(and its variant,} the equivalent chase \cite{Rocher16}) \footnote{\corvtwo{We could also consider} the \emph{parsimonious chase}, which was introduced in ~\cite{Leone:2012:ECD:3031843.3031846}, \corvtwo{but it is tailored for answering atomic queries and} does not compute a universal model of the KB, hence it is outside the family of chase variants studied here.}.   
All these chase variants compute logically equivalent results.
Nevertheless, they differ on their ability to detect logical redundancies possibly caused by the presence of nulls. Indeed, a factbase that contains nulls may be logically equivalent to one of its strict subsets, in which case we call it redundant. As deciding whether a factbase is redundant is computationally difficult (in fact, NP-complete \cite{DBLP:conf/stoc/ChandraM77}), a given chase variant may choose to detect only specific cases of redundancy.  
Note that, since redundancies can only be due to nulls, all chase variants output exactly the same results on ground factbases and rules without existential variables (i.e., Datalog rules, also called range-restricted rules \cite{AbiteboulHV}). 
On the other hand, the ability to detect redundancies has a direct impact on the termination of the chase. 
\corvtwo{In short,} the oblivious chase blindly performs all possible rule applications and terminates \corvtwo{less often than the other variants}, while the core chase produces factbases with no redundancies and terminates exactly when the KB admits a finite universal model. The rest of the chase variants lie between these two \corvtwo{extremes; they are} presented in the next section. 
The following example aims at illustrating the behavior of the oblivious and core chase variants. For \corvtwo{brevity}, universal quantifiers will be omitted.

\begin{example}\label{ex-bounded-2}
Take the ruleset $\mathcal R = \{p(x,y) \rightarrow \exists z. p(y,z)
\land p(z,y)\}$ and the factbase $F = \{\textit{p(a,b)}\}$.
Both chase variants perform the first rule application, which yields $F_1=\{p(a,b), p(b,z_0),
p(z_0,b)\}$. Then, two new  applications of $R$ are possible, one for each new atom $p(b,z_0)$ and $p(z_0,b)$. The first application would then add the atoms $p(z_0, z_1)$ and $p(z_1,z_0)$ to $F_1$,
yielding $F_2$. However, $F_2$ is logically equivalent to $F_1$ as there is a homomorphism from $F_2$ to $F_1$, which maps  $z_1$ to $b$ and \corvtwo{$z_0$ to itself}. 
\corvtwo{A similar thing happens} for the second rule application which would create the atoms $p(b,z_2)$ and $p(z_2,b)$. These are again redundant as $z_2$ can be mapped to $z_0$.
The oblivious chase simply performs both rule applications without testing for redundancy. This also means that it will go on forever, since each new (redundant) atom brings new rule applications. 
In contrast, the core chase detects that any rule application on $F_1$ yields an equivalent factbase, and outputs $F_1$. 
It is also worth mentioning that in this case the semi-oblivious chase will behave similarly as the oblivious chase, while the restricted chase will behave similarly as the core chase. 
\end{example}

Chase variants can be totally ordered with respect to the inclusion of the sets of knowledge bases on which they \corvtwo{terminate}:  \corvtwo{\[\text{Oblivious }<\text{ Semi-Oblivious }<\text{  Restricted }<\text{ Core}\]}
Here, $\text X_1 < \text X_2$ means that when $\text X_1$ halts on a KB, so does $\text X_2$, and there are KBs for which the converse is false.
Furthermore, the Skolem chase (respectively, the equivalent chase) terminate on the same KBs as the semi-oblivious chase (respectively, the core chase). 
Of~course, none of the chase variants terminates on a KB that does not admit a finite universal model as shown in Example \ref{ex-init}. 
 The termination problem, which asks whether for a given set of rules the chase will terminate on any factbase, is undecidable for all these chase variants 
 \cite{deutsch-nash-remmel08,blm10,GogaczM14}. 
 
Inspired by previous work on Datalog~\corvtwo{(see e.g., \cite{cgkv88,GaifmanMSV93}), } we study the related problem of \emph{boundedness}, which asks if, given a set of rules, the chase terminates on any factbase within a \emph{predefined} depth, i.e., independent from any factbase. The notion of depth is natural for Datalog programs whose evaluation is defined in a breadth-first manner, but requires further technical definitions when addressing several chase variants, which will be presented later. 
Hence, we will focus here on breadth-first chase variants. 
Intuitively speaking, a breadth-first chase is a process that,  starting from a factbase $F = F^0$, proceeds as follows: for each step $i > 0$ it (1) computes all new homomorphisms from rule bodies to $F^{i-1}$ and then (2) performs the rule applications associated with these homomorphisms, according to its own redundancy criterion, which yields $F^i$. 
The depth of a breadth-first chase on a given KB corresponds to the number of steps needed to terminate, i.e.\corvtwo{,} $k$ if the last computed factbase is $F^k$.

Given a chase variant X, we call a set of rules X-\emph{bounded} if there is a bound $k$ such that, for \emph{every} factbase, the X-chase stops within depth at most $k$. 
Of course, since chase variants differ with respect to termination they also differ with respect to boundedness, and each variant raises a distinct notion of boundedness.

Boundedness \corvtwo{implies} several nice properties. First, if a set of rules is X-bounded with bound $k$, then, for any factbase $F$,  the factbase obtained from $F$ after $k$ breadth-first X-chase steps is a universal model of the KB;  the converse is also true for X being the core chase {(or its variant the equivalent chase)}.  Moreover, boundedness also \corvtwo{implies} decidability of fundamental static analysis tasks on queries, i.e., data-independent problems  whose aim is to decide semantic properties of queries that can be exploited to optimize query answering. It ensures the first-order rewritability property \cite{dl-lite07} also known as  finite unification set property \cite{blms11}: 
 any conjunctive query $q$ can be rewritten using the set of rules $\mathcal R$ into a first-order query (and specifically a union of conjunctive queries) $Q$ such that, for any factbase $F$, the answers to $q$ on $(F, \mathcal R)$ are exactly the answers to $Q$ on $F$. Note that the conjunctive query rewriting procedure can be designed  in a such a way that it terminates within $k$ breadth-first steps, with $k$ the bound for the core chase \cite{LeclereMU16}. 
In turn, the first-order rewritability property ensures the decidability of conjunctive query containment under existential rules, which, given two conjunctive queries, asks if the set of answers to the first query is included in the set of answers to the second query, for any \corvtwo{fact}base. 


The importance of the boundedness problem has been recognized already for rules without existential variables. Indeed, the problem has been first posed and studied  for Datalog, where it has been shown to be undecidable \cite{HillebrandKMV95,Marcin1999}. 
 Example \ref{ex-Datalog1} illustrates some cases of bounded and unbounded rulesets in this setting.

\sloppy
\begin{example}\label{ex-Datalog1}
Consider the rulesets $\mathcal R_1=\{R\}$ and  $\mathcal R_2=\{R,R'\}$,
where
$R= \textit{p}(x,y) \land \textit{p}(y,z) \rightarrow \textit{p}(x,z)$ and
$R'= \textit{p}(x,y) \land \textit{p}(u,z) \rightarrow \textit{p}(x,z)$.
The set $\mathcal R_1$  contains a single transitivity rule for the
predicate $p$.
This set is clearly unbounded as, for any integer $k$, \corvtwo{there is} a
factbase \mbox{$F=\{p(a_i,a_{i+1})\ | \ 0\leq i<2^k  \}$} that requires  $k$ chase
steps.
On the other hand, $\mathcal R_2$ also contains a rule that joins
individuals on disconnected atoms. In this case, $i)$ if $R$
generates some facts, then $R'$ generates these same facts as well  and
$ii)$ $R'$ needs to be applied only at the first step, for any $F$, as it
does not produce any new atom at a later step. Therefore, $\mathcal R_2$ is
bounded with the bound $k=1$. Note that since these examples are in Datalog, the specificities
of the chase variants do not play any role.
\end{example}

Example \ref{ex-bounded-2} illustrates boundedness  beyond Datalog rules. The ruleset $\mathcal R$  is not bounded for the oblivious chase, which actually does not even terminate on the given KB. 
However, it can be checked that the core chase terminates with $\mathcal R$ on any factbase after at most one breadth-first step, hence $\mathcal R$ is bounded with $k = 1$ for the core chase.


\corvtwo{Even though} boundedness is undecidable already for Datalog, knowing that a set of rules is bounded for some chase variant does not help much \corvtwo{in practice if the bound remains unknown or is very large}. 
Hence, the goal of this  paper is to investigate decidability of the \emph{$k$-boundedness} problem which asks, for a given chase variant and an integer $k$, whether, for any factbase, the chase halts within depth $k$. 
The degree of boundedness of a ruleset (i.e., the smallest $k$ such that this ruleset is $k$-bounded) can be seen as a measure of the ``recursivity'' of a ruleset. 
Given a factbase $F$, it yields a polynomial bound on the size of the chase output on $F$ (i.e., with respect to the size of $F$). 
More precisely, the number of chase atoms produced at depth $d$ is exponential in $|\textit{body}|^d$, where $|\textit{body}|$ is the maximal size of a rule body. 
Similarly, given a conjunctive query $q$, $k$-boundedness yields a polynomial bound on the size of a rewriting of  $q$ as a union of conjunctive queries. 
Moreover, since the theoretical bound on the chase output considers all possible factbases, the actual number of breadth-first steps required by the chase on a given factbase is expected to be much smaller in practice. 

\smallskip
Our main contribution is to show that $k$-boundedness is indeed decidable for several main chase variants. It is worth noting that a general approach to derive decidability results for generally undecidable static analysis problems like boundedness is to restrict the rule language. Here, by focusing on the $k$-boundedness problem we are able instead to \corvtwo{obtain} results for the whole existential rule language.
Actually, we obtain a strong result by exhibiting a property that a chase variant may  \corvtwo{enjoy}, namely \emph{preservation of ancestry}, and prove that $k$-boundedness is decidable as soon as this property is satisfied.  We then show that it is the case for the oblivious, semi-oblivious and restricted chases, as well as all their breadth-first versions. 
\corvtwo{The decidability of   $k$-boundedness for the core chase 
remains an open question.  }

The paper is organized as follows. Section \ref{preliminaries} introduces preliminary notions, while Section \ref{kbound} defines $k$-boundedness parametrized by the considered chase variant  and states some fundamental properties of chase derivations on which we will rely to obtain our results. Section \ref{sec-k-bounded} presents the main results. 

\section{Preliminaries }\label{preliminaries}

This section is dedicated to basic notions and properties, including the formalization of the chase variants for which $k$-boundedness will be investigated. 

\subsection{Positive Existential Rules}\label{rules}
We consider a first-order setting  \corvtwo{without functional symbols (except constants, which can be seen as 0-ary functional symbols)} nor equality.
A vocabulary is a finite set of predicates (usually denoted with the letters $p,q,r$) and constants (usually denoted with $a,b,c$).
We assume also a countably infinite set of variables (denoted with $x,y,z$) used in formulas. 
A \emph{term} is either a variable or a constant. Each predicate $p$ is associated with a positive integer number, called the \emph{arity} of $p$. 
\corvtwo{An \emph{atom} is of the form $p(e_1,...,e_n)$, where $p$ is a predicate, $n$ is the \emph{arity} of $p$ and $e_1,...,e_n$ are terms.}

A \mbox{\emph{factbase}}, denoted by $F$, is an existentially quantified conjunction of atoms which is closed, i.e., every variable is existentially quantified. 
We will use the notation $\var(F)$, $\cnst(F)$, and $\term(F)$, to refer to the set of variables, constants, and terms that \corvtwo{occur} in $F$. 
It is very convenient to see factbases simply as sets of atoms.
So, for example, $\{p(a,x),q(x,b,c)\}$ can represent the existentially closed conjunction $\exists x.~ p(a,x)\wedge q(x,b,c)$.

A \emph{substitution} $\sigma$ is a mapping from a set of variables to a set of terms (usually represented as a set of single variable mappings, e.g. $\{x\mapsto a, y\mapsto b, ...\}$). 
A \emph{homomorphism} from a set of atoms $F$ to  $F'$ 
is a substitution \corvtwo{$h\colon\var(F)\rightarrow \term(F')$} such that $h(F)\subseteq F'$. It is known that a factbase $F$ logically entails a factbase $F'$ if and only if there exists a homomorphism from $F'$ to $F$ seen as atomsets (e.g., from \cite{DBLP:conf/stoc/ChandraM77}). 
An \emph{isomorphism} is a bijective homomorphism.
A subset $F'\subseteq F$ is a \emph{retract} of $F$ if there exists a substitution $\sigma$ that is the identity on the terms of $F'$ such that $\sigma(F)=F'$. In this case, $\sigma$ is also called a retraction  from $F$ to $F'$.
 A factbase $F$ is a \emph{core} if none of its strict subsets is a retract. 

An \emph{existential rule} $R$ is a first-order formula of the form \mbox{$\forall \bar x.~ \forall\bar y.~ B(\bar x,\bar y)\rightarrow \exists{\bar z}.~ H(\bar x, \bar z)$}, where $\bar x$, $\bar y$ and $\bar z$ are disjoint sets of variables, and $B$ and $H$ are conjunctions of atoms called the \emph{body} and the \emph{head} of the rule, respectively. 
The set of variables $\bar x$ is shared by the body and the head of the rule\corvtwo{; it is called} the \emph{frontier} of the rule, denoted by $\fr(R)$. The set $\bar z$ is called the set of \emph{existential variables} of the rule.
The set $\bar x\cup\bar y$ is the set of \emph{universally quantified variables} of $R$. 
An existential rule $R$ is \emph{Datalog} if it has no existential variables.  
We  call a set of existential rules,  denoted by $\mathcal{R}$, a \emph{ruleset}. 



In the following, universal quantifiers will be omitted in the examples.
Also, it will be sometimes convenient to consider a rule $R$ simply as a pair of sets of atoms $(B,H)$. 
Furthermore, we will use $body(R)$ to denote $B$ and $head(R)$ to denote $H$. 

\correction{Note that an existential rule is generally not equivalent to a clause, even if its head contains a single atom, because of its existential variables. 
However, any existential rule can be transformed into a set of Horn clauses 
by a Skolemization operation, }
which replaces each existentially quantified variable with a Skolem function whose arguments are the frontier variables of the rule. 
For instance, the  rule $R = p(x,w,y) \rightarrow \exists z_1.~ \exists z_2.~ q(x,z_1) \wedge t(z_1,z_2,y) $ is rewritten as two rules
$ p(x,w,y) \rightarrow q(x,f^{z_1}_R(x,y))$
and 
$ p(x,w,y) \rightarrow t(f^{z_1}_R(x,y),f^{z_2}_R(x,y),y)$. 
If moreover, all variables in the factbase are replaced \corvtwo{by fresh constants (0-ary Skolem terms)}, the transformation yields a  positive logic program 
which preserves the logical consequence with respect to all formulas where 
  the Skolem functions do not occur.  
 In Section \ref{sec-chases}, we will present a variant of the chase called ``Skolem chase'' which actually performs the forward chaining on the logical program associated with a set of existential rules.

 Let $F$ be a factbase and $R=(B,H)$ an existential rule. If there \corvtwo{is} a homomorphism $\pi$ from $B$ to $F$, \corvtwo{then} we say that $R$ is \emph{applicable} on $F$ via $\pi$. Then, the pair $\t=(R,\pi)$ is called a \emph{trigger}, and we will also say that $\t$ is {applicable} on $F$.  \correction{We denote by $\pi^s$ the extension of $\pi$ which maps all existential variables in $H$ to fresh variables indexed by the trigger $\t$. More precisely, for each existential variable $z$ in $H$, 
 we define $\pi^s(z)=z_t$.  
 The application of $\t$ on $F$ results in the factbase $F\cup\pi^s(H)$, which is called an \emph{immediate derivation} from $F$ through $\t$.} 
This fixed way to name fresh variables \corvtwo{ensures} that the chase always produces the same atoms when 
the same trigger is applied on different derivations. 
\corvtwo{It} will be useful when comparing forward chaining with the same ruleset but different factbases or chase variants. 
To simplify \corvtwo{notation} in cases where $R,\pi,B$ and $H$ are not specified, $\pi(B)$ is called the \emph{support} of $\t$ and is denoted by $\spp(\t)$ 
and  $\pi^s(H)$ is called the \emph{output} of $\t$ and is denoted by $\op(\t)$. 
Finally, given a trigger $(R,\pi)$
we denote by $\pi_{\mid\fr(R)}$ the restriction of $\pi$ to the frontier variables of $R$.

\subsection{Chase Variants}\label{sec-chases}

The chase is built upon the notion of derivation, which consists of the repeating application of rules from a certain ruleset to a factbase which is {evolving} with every rule application.

\sloppy
\begin{definition}[Derivation]\label{deriv}
A \emph{derivation}  from a knowledge base $(F,\mathcal{R})$ is a (possibly infinite) sequence of pairs 
$\mathcal{D}=(\emptyset, F_0),$ $(\t_1, F_1),(\t_2, F_2),$ $\dots$,
 where  $F_0=F$ and  \corvtwo{$F_{i}$, for each  $i>0$, is }an {immediate derivation} from $F_{i-1}$ through a new trigger $\t_i$ (i.e. $\t_i\neq\t_j$ for all \corvtwo{$i\ne j$}).
\end{definition} 

The sequence of triggers in a derivation is denoted by $\trg(\mathcal{D})$, 
while the set of atoms inferred by $\mathcal D$ is denoted by $F^{\mathcal D}=\bigcup_i F_i$. 
A derivation pair $(\t_i,F_i)$ is called an element of a derivation. 
The $k$-prefix of a derivation $\D$ is the prefix where the first $k$ triggers have been applied, and is denoted by $\mathcal D_{\mid k}$.\footnote{Therefore $\mathcal D_{\mid k}$ comprises the first $k+1$ elements of $\D$, since derivations start with $(\emptyset, F_0)$.}
In this case, we also say that $\mathcal D$ is an extension of $\mathcal D_{\mid k}$. Then, an atom $A$ is \emph{produced} by $\t_i\in\trg(\mathcal{D})$ if 
$i$ is the smallest integer such that 
$A\in F_i\setminus F_{i-1}$.

It is worth noting that if $\t$ produces $A$, then, of course, $A\in \op(\t)$, but 
 the converse may not hold. 
Indeed, there may be an atom in the output of a trigger that belongs to the initial factbase or that  has been produced earlier.
To illustrate, consider $F=\{p(a),q(a),r(a)\}$ and a ruleset $\mathcal R=\{R_1,R_2,R_3\}$, where $R_1:p(x)\rightarrow s(x)$, $R_2:q(x)\rightarrow s(x)$ and $R_3:p(x)\rightarrow r(x)$. These rules are applicable on $F$ with the respective triggers
  $\t_i=(R_i,\{x\mapsto a\})$. 
  Consider two derivations $\D_1,\D_2$ from $(F,\mathcal R)$ where $\t_i$ is applied first in $\D_i$, and then the remaining triggers.
Although the atom $s(a)$ is in the output of both $\t_1$ and $\t_2$, only
$\t_1$ produces $s(a)$ in $\D_1$ and only
$\t_2$ produces $s(a)$ in $\D_2$. Furthermore, $\t_3$ does not produce $r(a)$ in any derivation, as this atom belongs to $F$. 

\begin{definition}[Rank and Depth]
The \emph{rank} of an atom $A$ within a derivation $\mathcal D$ is defined as \mbox{${rank}\dD(A)=0$} if  
$A\in F_0$ and otherwise, let $\t$ be the trigger that produces $A$ in $\D$, then \mbox{${rank}\dD(A)=1+\max\{{rank}\dD(A')\ |\ A'\in \spp(\t)\}$.} 
This notion is naturally extended to triggers: \mbox{${rank}\dD(\t)=1+ \max\{rank\dD(A')\ | \ A'\in\spp(\t)\}$.}
Then, the \emph{depth} of a derivation is the maximal rank of its atoms if it is finite and infinite otherwise. 

\end{definition}

\sloppy
\noindent 
  
%

Informally speaking, the  atom rank does not indicate the number of triggers needed to produce it
but rather the number of parallel rule application steps that are needed to produce it. The notion of rank stems from breadth-first derivations
but  applies to any derivation. 
Importantly, two derivations may produce the same atom at different ranks, and this already occurs for Datalog knowledge bases, as illustrated by the next example.
\begin{example}\label{ex-Datalog2}
Let $F=\{p(a)\}$ and $\mathcal R=\{R_1,R_2,R_3\}$, where $R_1=p(x)\rightarrow q(x)$, $R_2=q(x)\rightarrow r(x)$, $R_3=p(x)\rightarrow r(x)$, and the following  derivations:
\[ \mathcal D_1 = (\emptyset,F), ((R_1,\pi), F_1), ((R_2,\pi), F_2), ((R_3,\pi) ,F_2)\]
\[ \mathcal D_2 = (\emptyset,F), ((R_1,\pi), F_1),  ((R_3,\pi),F_2), ((R_2,\pi),F_2)\]
where $\pi=\{x\mapsto a\}$. In both derivations $F_1 = \{p(a),q(a)\}$ and $F_2 = \{p(a),q(a),r(a)\}$. The atom $r(a)$ has rank 2 in $\mathcal D_1$, while it has rank 1 in $\mathcal D_2$. 
Note that both derivations are maximal (we will later call them terminating), however the depth of $\mathcal D_1$ is 2, whereas the depth of $\mathcal D_2$ is 1.
\end{example}


\medskip

We define a \emph{chase variant} as a class of derivations representing the possible runs of the chase.
The derivations that are proper to each variant are specified by imposing restrictions on which triggers can be applied and when they can be applied. 
 We start by presenting the oblivious (\OB), semi-oblivious (\SO), restricted (\RE), equivalent chase (\EQ), and then move to their breadth-first versions.
 \correction{In short, the oblivious chase applies all possible triggers once, 
 while the semi-oblivious chase does not apply triggers that map a rule frontier in the same way as a previously applied trigger.
 The restricted chase does not apply a trigger if there is a retraction from the resulting factbase to the current factbase, and 
 the equivalent chase does not apply a trigger if the resulting factbase is (logically) equivalent to the current factbase.
 The next example illustrates the behavior of these four chase variants.} 

 \begin{example}\label{ex-chase-var}
 
\smallskip
\correction
{
Consider the knowledge bases \mbox{$\mathcal K_1=(F,\{R_1\})$,} \mbox{$\mathcal K_2=(F,\{R_2\})$,} and \mbox{$\mathcal K_3=(F',\{R_3\})$} built from the factbases 
  \mbox{$F=\{\textit{p}(\textit{a},\textit{a})\}$} and
  \mbox{$F'=\{\textit{p}(\textit{a},  w)\}$}, where $w$ is a variable, 
    and the rules
}
\correction
{
  \mbox{$R_1=\textit{p}(x,y){\rightarrow}~ \exists z.  ~\textit{p}(x,z)$,}
  \mbox{$R_2=\textit{p}(x,y){\rightarrow}~ \exists z.  ~\textit{p}(y,z)$} and
  \mbox{$R_3=\textit{p}(x,y){\rightarrow}\exists z.  ~\textit{p}(x,x)\wedge\textit{p}(y,z)$.}
}

\correction
{
Regarding $\mathcal K_1$, the rule $R_1$ yields infinitely many triggers producing the atoms $\textit{p}(\textit{a},
z_0),$ $\textit{p}(\textit{a}, z_1),$ $\dots$ Hence, the oblivious chase does not halt~on~$\mathcal K_1$. 
Observe that all these triggers map the frontier variable $x$ to the same constant $a$. 
Hence, the semi-oblivious chase applies only the first trigger and halts. 
However, it does not halt on $\mathcal K_2$, while the
restricted chase does.
   Here again, $R_2$ yields infinitely many triggers, producing the atoms 
 $\textit{p}(\textit{a}, z_0),\textit{p}( z_0, z_1),\dots$; since each of them maps the frontier variables
    to new existentials, all these triggers are applied by the semi-oblivious chase.
     However, all generated atoms are redundant with respect to the initial atom
$\textit{p}(\textit{a},\textit{a})$. More precisely, there is a retraction from the factbase obtained after the first application of $R_2$ to the initial factbase, hence the restricted chase halts without producing any atom.
  On the other hand, the restricted chase does not halt on $\mathcal K_3$ while the
equivalent chase does. In this case, the first rule application yields the factbase
  $F''=\{\textit{p}(\textit{a},  w), 
\textit{p}(\textit{a}, \textit{a}), \textit{p}( w, z_0) \}$, where $w$ and $z_0$ are existentially quantified variables. Note that
there is no retraction from $F''$ to $F'$, and hence the restricted chase applies this trigger. The process continues\corvtwo{, because} from the atom $ \textit{p}( w, z_0)$ there is another trigger application which admits no retraction to $F''$. This creates a chain of atoms $\textit{p}(w, z_0),\textit{p}( z_0, z_1),\dots$ making the restricted chase  not terminating on $\mathcal K_3$. 
The equivalent chase terminates since, despite the fact that 
there is no retraction from $F''$ to $F'$, the  
factbase $F''$ is actually redundant and logically
  equivalent~to~$\{\textit{p}(\textit{a},\textit{a})\}$, which is also its core. Since any new trigger would produce equivalent factbases, the equivalent chase does not perform any other rule application and outputs $F''$ (the core chase, described later, would instead output $\{\textit{p}(\textit{a},\textit{a})\}$). 
Finally, note that $\textit{p}(\textit{a},\textit{a})$ is a (finite)
universal model for all knowledge bases $\mathcal K_1, \mathcal K_2,$ and $
\mathcal K_3$.
}
\end{example}

\begin{definition}[X-applicability]
\label{applicability}
Let $\mathcal{D}$ be a finite derivation from $(F,\mathcal{R})$ and $\t$ be a trigger which is applicable on $F^{\mathcal D}$.
Then,~$\t$~is  
 \begin{enumerate}
\item  \OB-\emph{applicable} on $\mathcal D$ if it does not belong to $\trg(\mathcal{D})$;
\item  \SO-\emph{applicable} on $\mathcal D$ if
$\t=(R,\pi)$ and
there exists no trigger $(R,\pi')\in\trg(\mathcal{D})$ such that  $\pi_{\mid\fr(R)}= {\pi'}_{\mid\fr(R)}$;
\item \RE-\emph{applicable} on $\mathcal D$ if there exists no retraction from $F\DD\cup \op(\t)$ to $F\DD$;
\item \EQ-\emph{applicable}  on $\mathcal D$ if there exists no homomorphism from $F\DD\cup \op(\t)$ to $F\DD$.
\end{enumerate} 
\label{monchv}
Let $\text{X} \in \{ \mathbf{O}, \mathbf{SO}, \mathbf{R},\mathbf{E} \}$.
A derivation $\mathcal D$ from a knowledge base $(F,\mathcal R)$ where every trigger
 $\t_i\in\trg(\mathcal{D})$ is X-applicable on the 
 prefix $\mathcal{D}_{\mid i-1}$~of~$\mathcal D$ is called an X\emph{-derivation.} The class of all the X-derivations is the X\emph{-chase}.
\end{definition}

Note that only the definitions of \OB- and \SO-applicability allow one to extend a derivation with a trigger that does not produce any (new) atom, which are instead ruled out by \RE- \corvtwo{and}  \mbox{\EQ-applicability.} 
The corresponding classes of derivations, i.e., chase variants, will be called
\mbox{$\mathbf{O}$-chase,} 
\mbox{$\mathbf{SO}$-chase,} 
\mbox{$\mathbf{R}$-chase,} and 
\mbox{$\mathbf{E}$-chase,} respectively.

A natural strategy for the chase is to proceed in a breadth-first manner, that is, by applying all possible triggers rank by rank. Breadth-first derivation are defined as follows.

\begin{definition}[Breadth-first Derivations]
An X-derivation $\mathcal D$ is said to be
\begin{itemize}
\item
\emph{rank-compatible} if $rank(\t_i)\leq rank(\t_j)$, for all 
$\t_i,\t_j\in\trg(\mathcal{D})$ such that $i<j$.
Whenever $rank(\t_i)< rank(\t_{i+1})$ the index $i$ is called a \emph{rank mark}\corvtwo{; it means} that the first $i$ triggers of $\D$ contain all and only the triggers of rank up to $rank(\t_i)$ of $\D$.
\medskip
\item \emph{breadth-first} if it is rank-compatible and, for every rank mark $i$,
the prefix $\mathcal D_{\mid i}$ cannot be extended to an X-derivation of the same depth. 
\end{itemize}
\end{definition}
Let $\text{X} \in \{ \mathbf{O}, \mathbf{SO}, \mathbf{R},\mathbf{E} \}$ be a chase variant.
The \mbox{$\mathbf{bf}$\textbf{-}X-chase} variant is the subclass of the \mbox{X-chase} comprised exclusively by breadth-first \mbox{X-derivations}. 
 As we will outline in the next section, breadth-first derivations not only 
  represent a natural way \corvtwo{of reasoning} on a knowledge base, but for some chase variants, they also behave better than the other derivations with regard to termination.

\begin{example} Consider the knowledge base from Example \ref{ex-Datalog2}: $\mathcal D_2$ is breadth-first, 
while $\mathcal D_1$ is not rank-compatible, as the rank of the third trigger is strictly smaller than the rank of the second trigger.
\end{example}

The notion of termination relies on \emph{fairness},
\corvtwo{ that is, on} the fact that, according to the chase variant, all applicable triggers have either been applied at some point or became redundant. 
Of course fair derivations may be infinite.
However, a (finite) fair derivation produces a (finite) universal model of the knowledge base.

%
%


\begin{definition}[Fairness and Termination]
An X-derivation $\mathcal D$ is \emph{fair} if whenever a trigger $\t$ 
is X-applicable on $\mathcal{D}_{\mid i}$
there exists a $k>i$ such that

\begin{itemize}
\vspace{-0.4cm}
\item[$\cdot$] either $\t_k=\t$,
\item[$\cdot$] or $\t$ is not X-applicable on $\mathcal{D}_{\mid k}$.
\end{itemize}
An X-derivation is  \emph{terminating} if it is both fair and finite. 
\end{definition}

Note that an X-derivation that is not terminating might be finite, but in this case it cannot be fair.
To conclude the section, let us now link the previous chase variants to some other known chase variants. The semi-oblivious and Skolem chases, both defined in \cite{marnette09}, lead to similar derivations. 
The Skolem chase consists \corvtwo{ in} running the oblivious chase on the Skolemized knowledge base.
This is obtained by  transforming the set of existential rules into a logical program, as described \corvtwo{in Section \ref{rules}}, and 
 by replacing every variable of the factbase with a fresh constant.
As already mentioned, this in turn corresponds to the classical forward chaining procedure for positive logic programs. 
The semi-oblivious and Skolem chase yield isomorphic results, in the sense that they generate exactly the same sets of atoms, up to a bijective renaming of nulls by Skolem terms. Therefore, we chose to focus on just one of them.

The core chase \cite{deutsch-nash-remmel08} and the breadth-first equivalent chase \cite{Rocher16} 
are two 
variants with a similar behavior, which terminate on a knowledge base if and only if this factbase has a finite universal model. 
The core chase proceeds in a breadth-first manner and, at each step, performs in parallel all rule applications according to the restricted chase criterion, \corvtwo{and} then computes a core of the resulting factbase. Hence, the core chase may remove at some step atoms that were introduced at an earlier step. 
After $i$ breadth-first steps, the equivalent chase and the core chase yield logically equivalent factbases, and they terminate on the same inputs.  
This follows from the fact that computing the core after each rule application or after a sequence of rule applications gives isomorphic results, and that 
two finite factbases are logically equivalent if and only if their respective cores are isomorphic.
However, it will be convenient to handle the equivalent chase from a formal point of view because it does not remove any atoms produced \corvtwo{by} a derivation.

\section{$k$-Boundedness}
\label{kbound}
In this section, after defining the notion of \corvtwo{{$k$-boundedness}}, we will present a set of properties of breadth-first derivations \corvtwo{relevant} to the chase variants that will be studied. This will allow us, \corvtwo{on the one hand,} to present some fundamental results that are key to the decidability of \mbox{$k$-boundedness} and, \corvtwo{on the other hand,} to provide a better understanding of boundedness itself.

\subsection{$k$-Boundedness}
As already mentioned, the concept of boundedness was first introduced for Datalog programs. A Datalog program is said to be \emph{bounded} if the number of breadth-first steps of a bottom-up evaluation of the program is bounded independently from any \corvtwo{fact}base (this notion being more precisely called \emph{uniform boundedness} to distinguish it from the notion of \emph{program boundedness} that restricts the set of predicates that may occur in the \corvtwo{fact}base)~\cite{GaifmanMSV93,Abiteboul89,GuessarianV94}. 
We apply this concept to the more general language of existential rules to define $k$-boundedness and parametrize it by the considered chase variants. 
Since every chase behaves differently with respect to termination, 
every chase gives rise to a distinct notion of $k$-boundedness. 
Indeed, a ruleset may be $k$-bounded for one chase variant but not $k$-bounded for another variant which employs a weaker applicability condition.
This is illustrated by the following example.


\begin{example} Consider  $\mathcal R=\{R_1,R_2\}$, where 
        $R_1 = p(x,y)\rightarrow \exists z.~ q(z,x)$
        and
        $R_2 = q(z,x)\rightarrow  \exists w.~ p(x,w)$.
\corvtwo{Then $\mathcal R$ is not $\OB$-bounded  but it is 
$\SO$-bounded; more precisely, $\mathcal R$ is 2-bounded for the \mbox{\SO-chase,} which means that all (fair) $\SO$-derivations built with $\mathcal R$ have depth at most  2.}
To see that, first note that, since the rule bodies have a single atom, 
the saturation of any factbase $F$ is 
\corvtwo{included in} the union of the saturations of each atom in $F$. 
Hence, the depth of a derivation on $F$ is bounded by the maximal depth of a derivation \corvtwo{from an atom in $F$.}
Then notice that every derivation from an atom can only alternate the application of rule $R_1$ with $R_2$. 
Now, to show that $\mathcal R$ is not $\OB$-bounded, take an initial fact $p(a,b)$ and observe that we can build an infinite \OB-derivation. On the other hand, the \mbox{\SO-chase} halts after producing the atom $q(z\tt1,a)$ at rank 1 (by applying $R_1$) and then $p(a,w\tt2)$ at rank 2 (by applying $R_2$). At this point $R_1$ is not \mbox{\SO-applicable} anymore because it has been already applied by mapping its frontier  $x$ to the constant $a$. 
\corvtwo{Indeed, $\SO$-derivations have depth at most two for all possible initial factbases. 
Finally, $\mathcal R$ is 1-bounded for the $\RE$-chase, i.e., all (fair) $\RE$-derivations built with $\mathcal R$ have depth at most  1.
Starting from the fact $p(a,b)$ for instance, the $\RE$-chase does not  produce $p(a,w\tt2)$ at rank~2 because of the retraction $\pi=\{w\tt2 \mapsto b\}$. 
It  follows that $\mathcal R$ is also 1-bounded for the  $\EQ$-chase. }
\end{example}

Another important dimension to the boundedness problem, studied also for chase termination \cite{DBLP:journals/fuin/GrahneO18}, is whether the bound on the chase depth is considered for all ({fair}) derivations or one ({fair}) derivation, as illustrated by the following example. 

\begin{example}
\label{exampleexists}
\correction{Consider the ruleset $\mathcal R_2=\{R,R'\}$ of Example~\ref{ex-Datalog1}. } 
Since $\mathcal R_2$ is Datalog, for every factbase $F$, all derivations from $(F,\mathcal R_2)$ are terminating.
Also, for every integer $k$, there exists a factbase $F$ and a terminating derivation from $(F,\mathcal R_2)$ of depth $k$, where $R$ is applied in all possible ways before applying $R'$.
Note, however, that rule $R'$ computes in one rank everything that $R$ can compute in many ranks. This means that, 
for every factbase $F$, there is also a terminating  derivation from $(F,\mathcal R_2)$ of depth~1 and the chase can halt earlier if a wise prioritization on rules is chosen.
\end{example}


This leads us to the following definition.
\begin{definition}[$k$-Boundedness]\label{kbaou}
Let  X be any chase variant. 
A ruleset $\mathcal R$ is 

\smallskip
\noindent 
  $\forall$-X-\emph{$k$-bounded} if, for each factbase $F$, all 
fair 
 X-derivations from $(F,\mathcal R)$ are of depth at most~$k$,

\smallskip
\noindent 
 $\exists$-X-\emph{$k$-bounded} if, for each factbase $F$, there is a fair X-derivation from $(F,\mathcal R)$ of depth at most~$k$.

\smallskip
\noindent
The $\forall$-X-\emph{$k$-boundedness} (resp. $\exists$-X-\emph{$k$-boundedness}) problem takes as input a ruleset $\mathcal R$ and an integer $k$ and  asks whether $\mathcal R$ is $\forall$-X-$k$-bounded (resp. $\exists$-X-$k$-bounded).

\end{definition}
Note that saying that 
all fair X-derivations from $(F,\mathcal R)$ are of depth at most $k$
is equivalent to say\corvtwo{ing} that 
all X-derivations from $(F,\mathcal R)$ are of depth at most $k$.
One direction of this property is trivial: if all X-derivations are \corvtwo{at most} $k$-deep then all fair X-derivations are \corvtwo{at most} $k$-deep.
For the other direction, we show the contrapositive.
Assume there is an unfair X-derivation of depth strictly greater than $k$.
Then take its shortest prefix of depth $k+1$ (i.e., which ends with the first trigger of rank $k+1$). 
We know that it can be extended by applying at least a trigger (but possibly an infinite number of them) so as to get a fair derivation of depth greater than $k$.
Also note that, by definition, a $k$-deep fair X-derivation is also terminating.

\medskip
It should be clarified that in this work we will focus our attention on $\forall$-X-$k$-boundedness, later  referred as X-$k$-boundedness or simply  $k$-boundedness when we do not need to specify any particular chase variant. However, we will establish a connection between the two versions of the problem that will allow us to transfer our results to $\exists$-X-$k$-boundedness for the (breadth-first)(semi-)oblivious chase variants (Theorem \ref{connection}).
 This leverages on some fundamental properties of breadth-first derivations, that are now presented.

\subsubsection*{Rank Minimality in Breadth-first Derivations}
The notion of the rank of an atom is central for studying $k$-boundedness, as the problem amounts to decid\corvtwo{ing} if  all atoms produced by the derivations of interest have  rank bounded by $k$, independently \corvtwo{of} the initial instance.

\smallskip
As illustrated by Examples \ref{ex-Datalog1}, \ref{ex-Datalog2} and \ref{exampleexists} on Datalog rulesets,
if a chase variant does not impose any constraint on the order in which triggers are applied\corvtwo{,} then the rank of an atom can vary from one derivation to another.
Of course, this happens also for rulesets that are not Datalog.
Take for example the knowledge base $(F,\{R_1,R_2,R_3\})$, where $F=\{p(a,b,c)\}$, \mbox{$R_1=p(x,y,w)\rightarrow \exists z.~ p(y,z,w)$,} \mbox{$R_2=p(x,y,w)\rightarrow p(y,y,w)\wedge  q(w)$} and \mbox{$R_3=q(w)\rightarrow t(w)$.} For any integer {$k\geq 2$}, there is an X-derivation generating the atom $t(c)$ at rank $k$, when $\text{X}\in\{\OB,\SO,\RE,\EQ\}$. Note that for the $\RE$\corvtwo{-} and $\EQ$\corvtwo{-}chase this derivation of arbitrary depth can also be terminating.


\smallskip
It is therefore natural to ask
what is the minimal rank that a given atom can assume in any X-derivation from a given knowledge base 
and also
 whether 
 there is a chase variant that allows one to produce all atoms at their minimal ranks, if a prioritization on triggers is assumed.
 It turns out that breadth-first oblivious derivations set the lower bound for the ranks of atoms. Recall also that this variant uses the weaker form of applicability condition which makes the $\bfO$-chase inferring any atom that can be produced by any other chase derivation of the same depth.

\begin{proposition}
\label{ranklowerbound}
Let $\mathcal D$ be any $\bfO$-derivation from $(F,\mathcal R)$ and $\D'$  any derivation from $(F,\mathcal R)$ of lower or equal depth. 
Then,  $F^{\mathcal D'}\subseteq F^{\mathcal D}$ and $rank_{\mathcal D}(A)\leq rank_{\mathcal D'}(A)$, for all $A\in F^{\mathcal D'}$.
\end{proposition}
\begin{proof}
By induction on the depth $m$ of $\D$. If $m=0$, then the claim follows as $F^{\mathcal D}=F=F^{\mathcal D'}$.
 Assume \corvtwo{that} $\D$ is of depth $m$. 
 We denote by $\D''$ the derivation obtained by applying the maximal subsequence of $\trg(\D')$ of depth up to $m-1$, in the given order. Then, let us denote by ${\mathcal D_{\mid depth(m-1)}}$ the maximal prefix of $\D$ of depth $m-1$. 
By \corvtwo{the} inductive hypothesis, $F^{\mathcal D''}\subseteq F^{\mathcal D_{\mid depth(m-1)}}$ and 
\corvtwo{$rank_{\mathcal D_{\mid depth(m-1)}}(A)\leq rank_{\mathcal D''}(A)$, for all $A\in F^{\mathcal D''}$}.
Each trigger $\t\in\trg(\D')\setminus\trg(\D'')$
 is \OB-applicable on $F^{\mathcal D_{\mid depth(m-1)}}$, because any atom in $\spp(\t)$ is produced within rank $m-1$, and so $\op(\t)\subseteq F^{\mathcal D}$. Therefore $F^{\mathcal D'}\subseteq F^{\mathcal D}$. Finally, for all $A\in F^{\mathcal D'}$, either $A\in F^{\mathcal D''}$ and then $rank_{\mathcal D}(A)=rank_{\mathcal D_{\mid depth(m-1)}}(A)\leq rank_{\mathcal D''}(A)= rank_{\mathcal D'}(A)$, or $A\in F^{\mathcal D'}\setminus F^{\mathcal D''}$ and then $rank_{\mathcal D'}(A)=m\leq rank_{\mathcal D}(A)$.
\end{proof}

It follows in particular that two $\bfO$-derivations of the same depth produce exactly the same set of atoms. More precisely, these derivations use exactly the same set of triggers but possibly taken in different order.
\footnote{
This does not hold for $\bfSO$-derivations which can produce atoms where fresh variables can be named differently. 
However, all semi-oblivious derivations of the same depth produce isomorphic results.
If a careful naming of the fresh variables upon the frontier variables of the rule applied by the trigger is chosen, \corvtwo{then} one can make all $\bfSO$-derivations \corvtwo{produce} the same sets of atoms. 
}

As breadth-first derivations apply all possible triggers at each rank (according to their criteria), 
it turns out that two breadth-first derivations of the same depth with the same knowledge base yield equivalent sets of atoms. Nevertheless, one derivation can produce more redundant atoms than the other. 
The following properties show that breadth-first oblivious derivations can be mapped by retractions to breadth-first semi-oblivious and restricted derivations of the same depth. Furthermore, these retractions map atoms to atoms of smaller or equal rank. 



\begin{proposition}
[Retraction 
on breadth-first derivations]
\label{babyretractionthe}
\label{retractionthe}
Let $\D$ be any derivation from $(F,\mathcal R)$ and $\D'$ a \bfX-derivation from $(F,\mathcal R)$ of equal or greater depth, with $\text{X} \in$ $\{\SO,\RE\}$. Then there exists a retraction $h$ from $F^{\mathcal D}\cup F^{\mathcal D'}$ to $F^{\mathcal D'}$ such that \corvtwo{$rank\dD(A)\geq rank_{\mathcal D'}\big(h(A)\big)$, for every $A\in F\DD$}.
\end{proposition}

\begin{proof}
\correction{We first show \corvtwo{the claim for \bfO-derivations $\D$.
In the last paragraph of the proof, we will use Proposition~\ref{ranklowerbound} to extend the claim to arbitrary derivations \D.}} 

\smallskip
We use induction on the depth of $\D$.
If $\D$ is of depth $0$, \corvtwo{then}  the claim follows as $h$ is the identity.
Assume that $\D$ is of depth $m$.
Let $\D_{\mid depth(m-1)}$ and $\D'_{\mid depth(m-1)}$ be the maximal prefixes of $\D$ and $\D'$ respectively that are of depth $m-1$.
By \corvtwo{the} inductive hypothesis, there is a retraction $h'$ from 
$F^{{\mathcal D}_{\mid depth(m-1)}}\cup F^{{\mathcal D}_{\mid depth(m-1)}'}$ to  $F^{{\mathcal D}_{\mid depth(m-1)}'}$.
Let \mbox{$\t=(R,\pi) \in \trg(\D) \setminus \trg(\D')$} be any trigger of rank $m$ in $\D$. Then the trigger \mbox{$\t'=(R,h'\circ \pi)$} is $\OB$-applicable on $F^{\mathcal D_{\mid depth(m-1)}'}$.

If $\t'$ is in $\D'$, then its output is in $\D'$, and we can extend $h'$ to $\op(\t)$ with a bijective renaming of the existential variables in $\op(\t)$, \corvtwo{that is, $h'(x_{\tt{}})=x_{\mathtt{t}'}$,} for every existential variable  $x$ that \corvtwo{occurs} in $R$. In this case the atoms of $\op(\t)$ are mapped to atoms of same rank in $\D'$.

If $\t'$ is not in $\D'$, \corvtwo{then} it is not $\RE$-applicable on $F^{\mathcal D_{\mid depth(m-1)}'}$ (note that a trigger which is not \SO-applicable is not \RE-applicable either).
Hence, there is a retraction $h''$ from $\op(\t')\cup F^{\mathcal D_{\mid depth(m-1)}'}$ to $F^{\mathcal D_{\mid depth(m-1)}'}$. So we extend $h'$ by mapping the existential variables in $\op(\t)$ to $h''(\op(\t'))$, \corvtwo{that is, \mbox{$h'(x_{\tt{}})=h''(x_{\mathtt{t}'})$,}} for every existential variable  $x$ that appears in $R$.

Notice that in both cases, the atoms of rank $m$ in $\D$ are mapped to atoms of equal or lower rank in $\D'$. Since the sets of existential variables produced by all triggers of rank $m$ in $\D$ are disjoint, the union of all extensions of $h'$ for all these triggers is
a retraction of $F^{{\mathcal D}_{\mid depth(m)}}\cup F^{{\mathcal D}_{\mid depth(m)}'}$ to $F^{{\mathcal D}_{\mid depth(m)}'}$. 

\smallskip
If now $\D$ is any derivation, \corvtwo{then} we go back to the previous case by using Proposition~\ref{ranklowerbound}, which asserts the existence of a \bfO-derivation $\D^*$ of the same depth as $\D$ such that $F\DD\subseteq F^{\mathcal D^*}$ and \corvtwo{$rank\dD(A)\geq rank_{\mathcal D^*}(A)$, for all $A\in F\DD$.}
Since $F\DD\subseteq F^{\mathcal D^*}$,
the retraction $h$ from $ F^{\mathcal D^*}\cup F^{\mathcal D'}$ to $F^{\mathcal D'}$
is also a retraction from $F^{\mathcal D}\cup F^{\mathcal D'}$ to $F^{\mathcal D'}$ \corvtwo{and we also have \mbox{$rank\dD(A)\geq rank_{\mathcal D'}\big(h(A)\big)$,} for every $A\in F\DD$}.
\end{proof}

Finally, we show that all breadth-first derivations for the oblivious, semi-oblivious and restricted chase agree on the rank of the atoms that they produce, independently from the variant.

\begin{proposition}
\label{stable-rank}\label{bfminrank}
Let $\D$ be a $\mathbf{bf}$\textbf-X-derivation from $(F,\mathcal R)$ and 
$\D'$ be a $\mathbf{bf}$\textbf-Y-derivation from $(F,\mathcal R)$ 
with \mbox{$\text{X},\text{Y}\in \{\OB,\SO,\RE\}$.}
Then $rank_{\mathcal D}(A)=rank_{\mathcal D'}(A)$, for all $A\in F^{\mathcal D}\cap F^{\mathcal D'}$.
\end{proposition}

\begin{proof} By Proposition \ref{ranklowerbound}, the property holds for $\bfO$-derivations, because such derivations produce the same atoms at exactly the same rank.
To prove the statement it suffices to show that  for any \bfX-derivation $\D$, with $\text{X}\in$ $\{\SO, \RE\}$, there is a $\bfO$ derivation $\D'$ from $(F,\mathcal R)$  and 
$rank_{\mathcal D}(A)= rank_{\mathcal D'}(A)$ for all $A\in F^{\mathcal D}\cap F^{\mathcal D'}$.
Assume that $\D$ is of depth $m$.
Let $\D'$ be any $\bfO$-derivation of depth at least $m$.
By Proposition \ref{ranklowerbound}, $F^{\mathcal D'}\subseteq F^{\mathcal D}$.
We now proceed by contradiction.
Let $\t$ be the first trigger in $\trg(\D)$ that produces an atom $A$ (common to $F^{\mathcal D}$ and $F^{\mathcal D'}$) such that $rank\dD(A)\not=rank_{\mathcal D'}(A)$. 
Suppose that $rank\dD(\t)=m$. Thus for all common atoms of rank $i< m$, their ranks coincide between the two derivations. By Proposition \ref{ranklowerbound}, 
$rank\dD(A)>rank_{\mathcal D'}(A)$ as $\bfO$ derivation set the lower bound for the rank of atoms. 
Moreover, every atom in $\spp(\t)$ is produced within rank $m-1$ in $\D$, so  
every atom in $\spp(\t)$ is produced within rank $m-1$ also in $\D'$.
Let us denote by $\mathcal D_{\mid depth(m-1)}$ and $\mathcal D_{\mid depth(m-1)}'$ the maximal prefix of $\D$ and $\D'$ of depth $m-1$.
We conclude that $\t$ is X-applicable on ${\mathcal D_{\mid depth(m-1)}}$ and $\OB$-applicable on ${\mathcal D_{\mid depth(m-1)}'}$.

As $\t$ does not produce $A$ at rank $m$ in $\D'$, there is a trigger $\t'\in\trg(\D')$ producing $A$ at a rank smaller than $m$, therefore $A\in F^{\mathcal D'_{\mid depth(m-1)}}$.
Since $A$ belongs to the output of two different triggers $\t$ and $\t'$ in $\D'$,
and fresh variables are named after the trigger that generates them,
we conclude that $A$ is obtained from an atom in the rule head of $\t$ (and $\t'$) that only uses frontier variables.
Therefore all the terms of $A$ are present in $\spp(\t)$, so also in $F^{\mathcal D_{\mid depth(m-1)}}$.

By Proposition \ref{babyretractionthe}, 
there \corvtwo{is} a retraction $h$ from $F^{\mathcal D'_{\mid depth(m-1)}}\cup F^{\mathcal D_{\mid depth(m-1)}}$ 
to $F^{\mathcal D_{\mid depth(m-1)}}$. Because all the terms of $A$ appear in $F^{\mathcal D_{\mid depth(m-1)}}$, we \corvtwo{obtain} $h(A)=A$. Hence \mbox{$A\in F^{\mathcal D_{\mid depth(m-1)}}$,} so $A$ is not produced by $\t$ in $\D$. A contradiction.
\end{proof}


Proposition \ref{stable-rank} 
will  be essential to prove decidability of $k$-boundedness for the breadth-first  semi-oblivious and restricted chase. 
Surprisingly, it does not hold for the breadth-first equivalent chase variant, i.e., with $\text X =\text Y = \bfE$, as atoms can have different ranks depending on the derivation, as shown by the following example.

\begin{samepage}
\begin{example}
For two predicates $p_1,p_2$, we denote \corvtwo{by} $R_{p_1p_2}$ the unary inclusion rule $p_1(x)\rightarrow p_2(x)$. Let  $\mathcal R=\{R_{pq},R_{qp},R_{qr},R_{pr},R_{rs},R_{rp}\}$ and $F=\{p(z_1), q(z_2)\}$, where $z_1$ and $z_2$ are variables.
We construct two 
\mbox{$\bfE$-derivations $\D_1,\D_2$} which produce the atom $p(z_2)$ at different ranks.
Let $h_1=\{x\mapsto z_1\}$ and  $h_2=\{x\mapsto z_2\}$. We describe the derivations in terms of their triggers.
 
 \medskip\samepage{
 $
 \trg(\D_1)=(R_{pq},h_1),(R_{qr},h_2),(R_{pr},h_1),(R_{rs},h_2),(R_{rp},h_2)
 $
 
 $
  \trg(\D_2)=(R_{qp},h_2),(R_{qr},h_2),(R_{rs},h_2)
 $}
 
 \smallskip\noindent
The first rank of $\D_1$ sees the application of three rules producing the atoms $q(z_1)$, $r(z_2)$, and $r(z_1)$ in the given order. At this point the trigger $(R_{qp},h_2)$ producing $p(z_2)$ is $\RE$-applicable but not $\EQ$-applicable, because the homomorphism $h=\{z_2\mapsto z_1\}$ 
from $F^{\mathcal D_{1\mid 3}}\cup\{p(z_2)\}$ to $F^{\mathcal D_{1\mid 3}}$
(recall the notation $\D_{\mid i}$ for the application of the first $i$ triggers of $\D$)
makes this inference redundant. 
Rank 1 is thus complete. At rank 2, two rules starting from the atom $r(z_2)$ are applied. \corvtwo{The atom $s(z_2)$ is produced first, resulting to $h$ \emph{not} being homomorphism 
from $F^{\mathcal D_{1\mid 4}}\cup\{p(z_2)\}$ to $F^{\mathcal D1_{\mid 4}}$. So then we can apply $R_{rp}$ and produce $p(z_2)$. The derivation $\mathcal D_1$ is terminating and the final factbase is $\{p(z_1), q(z_2),q(z_1),r(z_2),r(z_1),s(z_2),p(z_2)\}$. 
But in $\mathcal D_2$ the triggers are applied in a different order, starting from one that produces $p(z_2)$ at rank 1. This derivation also terminates at rank 2 with factbase $\{p(z_1), q(z_2),p(z_2),r(z_2),s(z_2)\}$. Of course the factbases produced by the two derivations are equivalent. }
\end{example}\end{samepage}

\subsubsection*{Breadth-first Derivations: Termination and Depth}

%
%

\corvtwo{We have shown that when $\text X \in \{\OB,\SO,\RE\}$, a breadth-first X-derivation produces all atoms at their lower rank.
Does this imply that, if the X-chase halts, then the $\mathbf{bf}$-X-chase halts as well?}
This actually holds for the oblivious and semi-oblivious chases. Indeed, among all terminating derivations for the oblivious and semi-oblivious chase, the breadth-first are the ones with the smallest depth, as stated by the next proposition, which directly follows from the above results. 

\begin{proposition}
\label{osoe}
Let  $\text{X}\in\{\OB,\SO\}$.
For each terminating X-chase derivation from $(F,\mathcal R)$, there is a terminating $\mathbf{bf}$\textbf{-}X-derivation from $(F,\mathcal R)$
of smaller or equal depth.
Moreover, all terminating \bfX-chase derivations from $(F,\mathcal R)$ have the same depth.
\end{proposition}  
\begin{proof}
The result directly follows from Propositions \ref{ranklowerbound} and \ref{stable-rank} for the \bfO ~variant. For the \bfSO ~variant we furthermore argue that all \bfSO-derivations of the same depth produce isomorphic results. 
\end{proof}

The case of the restricted chase is more complex, since, {for a given factbase}, some fair derivations may terminate, while others may not. It may happen that all breadth-first derivations terminate (even with depth less than a predefined number $k$), but there is a fair non-breadth-first derivation that does not terminate. It may also be the case that no breadth-first derivation terminates, but there is a non-breadth-first derivation that terminates (even with predefined depth less than $k$), as illustrated by the next example.

\begin{example}\label{example20} Let $F=\{p(a,b)\}$ and $\mathcal R=\{R_1,R_2,R_3\}$ with $R_1=p(x,y)\rightarrow \exists z.~ p(y,z)$, \mbox{$R_2=p(x,y)\rightarrow\exists z.~ q(y,z)$} and $R_3=q(y,z)\rightarrow p(y,y)$. It is easy to see that a breadth-first \RE-chase derivation in this knowledge base cannot be terminating. However by applying only $R_2$ on $F$ and then $R_3$ on the new atom, we obtain a terminating \RE-chase derivation. {Note also that, for any factbase, there is a terminating \RE-chase derivation of depth at most 2.}
\end{example}

\noindent Hence, in the case of the restricted chase, breadth-first derivations are not necessarily derivations of minimal depth. In contrast, this holds for rank-compatible $\RE$-chase derivations 
(which differ from breadth-first derivations because not all active triggers of a given rank are necessarily applied). 
\corvtwo{\begin{proposition}\label{res}
For each terminating $\mathbf{R}$-derivation from $(F,\mathcal R)$,
there exists a terminating rank-compatible \mbox{$\mathbf{R}$-derivation} from $(F,\mathcal R)$ 
of smaller or equal depth.
\end{proposition}}

\noindent
\begin{proof} 
Let $\mathcal{D}$ be a terminating $\mathbf{R}$-derivation from $(F,\mathcal{R})$. 
Let $\mathcal{T}$ be a sorting of $\trg(\D)$ such that the rank of each element is greater or equal to the rank of its predecessors. 
Let $\mathcal{D}'$ be the maximal $\RE$-derivation 
from $(F,\mathcal R)$
that only uses  the triggers of $\mathcal T$, in the given order.
\corvtwo{Now, if $\trg(\D')=\trg(\D)$, then $\D'$ is rank-compatible; moreover, $\D$ is terminating, hence $
\D'$ is also terminating, because both derivations produce the same atomsets.
Otherwise, $\trg(\D')\neq\trg(\D)$ is due to some of triggers occurring in $\D$ that are no longer \mbox{\RE-applicable} in $\mathcal{D}'$. We will show that $\D'$ is a terminating \mbox{\RE-derivation} in this case as well. }\smallskip
\\
\indent
Let $\t_1\dots \t_m$  be the triggers of $\trg(\D')\setminus\trg(\D)$ that were not \mbox{\RE-applicable} when constructing $\mathcal{D}'$, where $\t_i=(R_i,\pi_i)$ and $R_i=H_i\rightarrow B_i$.
So \mbox{$F^{\mathcal D}=F^{\mathcal D'}\cup \pi_1^s(H_1)\cup\cdots\cup\pi_m^s(H_m)$.}
We know that for every $i\in\{1,...,m\}$, there is a retraction $\sigma_i$ from 
  ${\pi_i^s(H_i)}\cup {F^{\mathcal D'}}$ to ${F^{\mathcal D'}}$.
Since with $\sigma_1,...,\sigma_m$, only new variables are mapped to different terms (and all other variables are mapped to themselves), we can define the substitution $\tau=\bigcup_{i=1}^m \sigma_i$ which has the property that $\tau\big(F^{\mathcal D'}\cup \pi_1^s(H_1)\cup\cdots\cup\pi_m^s(H_m)\big)=F^{\mathcal D'}$.\smallskip
\\
 \indent
Suppose that there is a new trigger $\t=(R,\pi)$ with $R=B\rightarrow H$
which is \RE-applicable on $\mathcal{D}'$ \corvtwo{, thus there is a homomorphism from $B$ to $F^{\mathcal D'}$}. 
Then, since $F^{\mathcal D'}\subseteq F^{\mathcal D}$, 
\corvtwo{there is a homomorphism from $B$ to $F^{\mathcal D}$}. But because $\mathcal{D}$ is a terminating \mbox{\RE-derivation,} $\t$ is not \RE-applicable on $\D$. 
We conclude that there is a retraction $\sigma$ from $\pi^s(H)\cup F^{\mathcal D}$ to $F^{\mathcal D}$.
Since the new variables created from $\t_1,\dots,\t_m,\t$ are all different,  the composition $\tau\circ\sigma$ is  a retraction from 
$\pi^s(H)\cup F^{\mathcal D'}$ to $F^{\mathcal D'}$,  which leads us to conclude that no such trigger $(R,\pi)$ is \mbox{\RE-applicable} on $\D'$. \corvtwo{This is} a contradiction.
\end{proof}\smallskip
\\
\noindent\corvtwo{
However, note that different terminating rank-compatible $\RE$-derivations (and so also \mbox{\bfR-derivations)} from the same knowledge base may have different depths (see  \mbox{Example~\ref{example20}).}}
Thus, an important difference between the breath-first (semi-)oblivious and restricted chase variants is that although all terminating breadth-first restricted derivations agree on the ranks of the common atoms they produce (Proposition \ref{stable-rank}), they are not guaranteed to have the same depth.

\medskip
Putting everything together, we obtain a characterization of 
$\exists$-X-$k$-boundedness
in terms of 
$\forall$-X-$k$-boundedness for the oblivious and semi-oblivious chases, as well as their breadth-first variants.

\begin{proposition}
\label{connection}
Let X $\in\{\OB,\SO\}$ and let $\mathcal R$ be a ruleset. Then the following statements are equivalent: 
\begin{enumerate}
\item $\mathcal R$ is $\exists$-X-$k$-bounded;
\item $\mathcal R$ is  $\exists$-$\mathbf{bf}$-X-$k$-bounded;
\item $\mathcal R$ is   $\forall$-$\mathbf{bf}$-X-$k$-bounded;
\end{enumerate}
\end{proposition}
\begin{proof}
Follows from Proposition \ref{osoe}.
\end{proof}

In the next section we focus on $\forall$-k-boundedness and show that it is decidable for the oblivious and semi-oblivious chase as well as their breadth-first versions. \corvtwo{By Proposition \ref{connection}, this will in turn imply  decidability of $\exists$-k-boundedness for those chase variants. }

\section{Decidability of $k$-boundedness in Chase Variants}\label{sec-k-bounded}
%

Our approach to the study of $k$-boundedness is based on identifying a property of chase variants that implies decidability of the problem, namely \emph{preservation of ancestry}. This property is \corvtwo{satisfied} by the oblivious, semi-oblivious and restricted chase, as well as their breadth-first variants, which implies decidability of $k$-boundedness for all these variants. The interest of this approach is that we abstract away from a particular variant thereby providing a proof schema for decidability that can be applied to all chases preserving ancestry. Moreover, we identify a stronger property, namely \emph{heredity}, which is not \corvtwo{enjoyed} by the breadth-first semi-oblivious and breadth-first restricted chases. While preservation of ancestry is sufficient \corvtwo{for} decidability of $k$-boundedness, it remains interesting to consider heredity, as it leads to simpler proofs. 
The next table summarizes the results obtained concerning these properties.


\begin{table}[h!]
        \centering           
\begin{longtable}{|p{0.23\textwidth}|p{0.08\textwidth}|p{0.08\textwidth}|p{0.08\textwidth}|p{0.08\textwidth}|p{0.08\textwidth}|p{0.08\textwidth}|p{0.07\textwidth}|p{0.08\textwidth}|}
\hline \vspace{-1em}
  & 
 \vspace{-1em}\begin{equation*}
\OB
\vspace{-1em}\end{equation*} & 
\vspace{-1em}\begin{equation*}
\bfO
\vspace{-1em}\end{equation*} & 
\vspace{-1em}\begin{equation*}
\SO
\vspace{-1em}\end{equation*} & 
\vspace{-1em}\begin{equation*}
\bfSO
\vspace{-1em}\end{equation*} & 
\vspace{-1em}\begin{equation*}
\RE
\vspace{-1em}\end{equation*} & 
\vspace{-1em}\begin{equation*}
\bfR
\vspace{-1em}\end{equation*} & 
\vspace{-1em}\begin{equation*}
\EQ
\vspace{-1em}\end{equation*} & 
\vspace{-1em}\begin{equation*}
\bfE
\vspace{-1em}\end{equation*}  \\
\hline \vspace{-2mm}
\begin{center} Heredity\end{center}\vspace{-4mm} & \vspace{-4mm}\begin{center}
$\displaystyle \checkmark $
\end{center}\vspace{-4mm}
 & \vspace{-4mm} \begin{center}
$\displaystyle \checkmark $
\end{center}\vspace{-4mm}
 & \vspace{-4mm}\begin{center}
$\displaystyle \checkmark $
\end{center}\vspace{-4mm}
 & \vspace{-4mm}\begin{center}
 \xmark 
\end{center}\vspace{-4mm}
 & \vspace{-4mm}\begin{center}
$\displaystyle \checkmark $
\end{center}\vspace{-4mm}
 & \vspace{-4mm}\begin{center}
 \xmark 
\end{center}\vspace{-4mm}
 & \vspace{-4mm}\begin{center}
 \xmark 
 \end{center}\vspace{-4mm}
 & \vspace{-4mm}\begin{center}
 \xmark 
\end{center}\vspace{-4mm}
  \\  
\hline  \vspace{-4mm}\begin{center}
 \resizebox{3cm}{!}{Preservation of Ancestry} \end{center}\vspace{-4mm}& \vspace{-4mm}\begin{center}
$\displaystyle \checkmark $
\end{center}\vspace{-4mm}
 & \vspace{-4mm}\begin{center}
$\displaystyle \checkmark $
\end{center}\vspace{-4mm}
 & \vspace{-4mm}\begin{center}
$\displaystyle \checkmark $
\end{center}\vspace{-4mm}
 & \vspace{-4mm}\begin{center}
$\displaystyle \checkmark $
\end{center}\vspace{-4mm}
 & \vspace{-4mm}\begin{center}
$\displaystyle \checkmark $
\end{center}\vspace{-4mm}
 & \vspace{-4mm}\begin{center}
$\displaystyle \checkmark $
\end{center}\vspace{-4mm}
 & \vspace{-4mm}\begin{center}
 \xmark 
\end{center}\vspace{-4mm}
 & \vspace{-4mm}\begin{center}
 \xmark 
\end{center}\vspace{-4mm}
   \\  
 \hline
\end{longtable}
\caption{Chase variants with respect to heredity / preservation of ancestry. 
}       
\end{table}

\smallskip
To decide if a ruleset $\mathcal R$ is $k$-bounded, 
we consider the dual problem of determining if there 
\corvtwo{is} a factbase and a derivation of depth $k{+}1$ constituting a counterexample to the property.
Our method relies on the construction of a  special
factbase that allows us to reproduce such \corvtwo{a} derivation when the chase enjoys either preservation of ancestry or heredity. We show that the size of such a factbase depends on $k$ and $\mathcal R$ only. Therefore, by testing the chase depth over a representative set of \corvtwo{bounded-size} factbases, we obtain decidability of the problem.

\subsection{Preservation of Ancestry}

\medskip
Preservation of ancestry is a notion built on the ancestors set of an atom produced by a derivation.

\begin{definition}[Chase Graph and Ancestors]\label{chsgrph}
Let $\mathcal{D}$ be a derivation from $(F,\mathcal{R})$. 
The \emph{chase graph} of $\mathcal D$ is a (possibly infinite) directed graph $G=(V,E)$, 
where the nodes are the atoms  in $F^{\mathcal D}$, 
and there is an edge from $A'$ to $A$ labeled with $\t$
whenever $A$ is produced by $\t\in \trg(\D)$ and $A'\in\spp(\t)$.
Moreover, we say that an atom $A'$ is an \emph{ancestor} of $A$ if there is a non-empty path from $A'$ to $A$ in the chase graph. 
The set of ancestors of an atom $A$ is denoted by $\textit{Anc}\dD(A)$, while $\textit{Anc}^k_{\mathcal D}(A)$ is the subset of ancestors whose rank is exactly $k$. If $A \in F$, then $\textit{Anc}\dD(A) = \emptyset$. Both notations are extended to sets of atoms.  
\end{definition}
Note that an edge is labeled by the \emph{first} trigger of the derivation that produces the atom.
Also, the {rank} of an atom is equal to the maximum length of a path to this atom in the chase graph, and the {depth} of the derivation is equal to the maximum length of a path in the chase graph.
Note that $Anc_{\mathcal D}^0(A)$ are the ancestors of $A$ that belong to the initial factbase, \corvtwo{which} we call \emph{prime ancestors}. 
Importantly, the set of prime ancestors of an atom is finite  and its size depends on the rank of the atom itself, as well as on the maximum size (that is, the number of atoms) of the body of a rule, hence it can be bounded.

\begin{proposition}
\label{anclue}
Let $\mathcal{D}$ be a derivation from $(F,\mathcal R)$ and $A\in F^{\mathcal D}$ an atom of rank $k > 0 $. 
Then $|Anc_{\mathcal{D}}^0(A) |\leq b^k$, where $b=max\big\{\ |body(R)| : R\in\mathcal R\big\}$.
\end{proposition}
\begin{proof}
Let $(R, \pi)$ be the trigger that produces $A$. Then $Anc_{\mathcal{D}}^0(A)$ is equal to the union of $\{\pi(B_i)\} \cup Anc_{\mathcal{D}}^0(\pi(B_i))$ for all atoms $B_i \in body(R)$. 
The proof follows by a simple induction on $k$.  
%
%
\end{proof}

%
We denote by $\D_{\mid G}$ the \emph{restriction of a derivation} $\D$ from $(F,\mathcal R)$ with respect to $G\subseteq F$, which is the maximal derivation 
from $(G,\mathcal R)$
that only uses  the triggers of $\trg(\D)$, in the given order.
As stated below, it can be easily verified that prime ancestors are enough to infer a certain atom
as well as all its ancestors.

\medskip
\begin{proposition}
\label{obliviancest}
Let $\mathcal{D}$ be a derivation from $(F,\mathcal R)$ and $A\in F^{\mathcal D}\setminus F$. 
Let $\mathcal D' = \D_{\mid Anc_{\mathcal D}^0(A)}$. Then
\mbox{$Anc_{\mathcal D}(A) \cup \{A\} \subseteq F^{\mathcal D'}$. }
\end{proposition}

\medskip
\begin{proof}
By induction on the rank of $A$ in $\D$. 
Let $\t \in\trg(\D)$ be the trigger that produces $A$. \corvtwo {If $rank\dD(A)=1$, then 
 $Anc_{\mathcal D}^0(A)=\spp(\t)$}, hence $\t \in\trg(\D')$ and the claim follows. Assume \corvtwo{now} that $rank\dD(A)=n+1$.
Then
\mbox{$Anc_{\mathcal D}(A)=Anc_{\mathcal D}(\spp(\t))\cup \spp(\t)$} and $Anc^0_{\mathcal D}(A)=Anc^0_{\mathcal D}(\spp(\t))$.
Since  $\spp(\t)$ contains only atoms of rank up to $n$, by the inductive hypothesis, $Anc_{\mathcal D}(\spp(\t))\cup \spp(\t) \subseteq F^{\mathcal D'}$,
hence $\t \in\trg(\D')$ and $A \in F^{\mathcal D'}$. 
\end{proof}

\medskip
Two issues have to be addressed at this point if one wants to exploit the restriction of a derivation to decide $k$-boundedness. The first is that 
 $\D$ and $\D_{\mid Anc_{\mathcal D}^0(A)}$ could belong to different chase variants.
 The second is that $\D$ and $\D_{\mid Anc_{\mathcal D}^0(A)}$ could disagree on the rank of atoms.
Both properties are not immediate, and this is especially true when $\D$ is breadth-first as the restriction of a derivation could, \corvtwo{on the one hand}, break rank exhaustiveness 
 and, \corvtwo{on the other hand}, increase the rank of some atoms. 


The following examples illustrate that for the breadth-first semi-oblivious and breadth-first restricted chase the restriction of a derivation can lead to a different chase variant.
The examples show that there may be triggers that were not applicable starting from $F$ but become applicable by starting from the prime ancestors of an atom.

\medskip
\begin{example} [Breadth-first  semi-oblivious chase] \label{bf-so-ancestry-ex}
Let $F = \{p(a,b), r(a,c)\}$ and $\mathcal R = \{R_1= p(x,y) \rightarrow r(x,y),$ \mbox{$R_2= r(x,y) \rightarrow \exists z.~ q(x,z),$} \mbox{$R_3=r(x,y) \rightarrow t(y)\}$.} Let $\mathcal{D}$ be the (terminating)
$\bfSO$-chase 
derivation of depth $2$ from $F$ whose sequence of associated triggers is $(R_1, \pi_1), $ $(R_3,\pi_2), $ $(R_2,\pi_2),(R_3,\pi_1)$ with \mbox{$\pi_1=\{x\mapsto a, y\mapsto b\}$} and \mbox{$\pi_2=\{x\mapsto a, y\mapsto c\}$}, which produces $r(a,b),  t(c), q(a,z_{(R_2,\pi_2)}), t(b)$; the trigger $(R_2,\pi_1)$ is then \OB-applicable but not \mbox{\SO-applicable,} as it maps 
\corvtwo{the frontier variable $x$ to $a$, like the trigger $(R_2, \pi_2)$}. Let  $F' =\{ p(a,b)\}$. The restriction of $\mathcal{D}$ induced by $F'$ includes only $(R_1,\pi_1), (R_3,\pi_1)$ and is a \SO-chase derivation of depth $2$, however,  it is not breadth-first since now $(R_2,\pi_1)$ is \SO-applicable at rank $2$ (thus the breadth-first condition is not satisfied). 
\end{example}

\begin{example} [Breadth-first   restricted chase] \label{bf-r-ancestry-ex}

\noindent
Let $F = \{p(a,b), q(a,c)\}$ and $\mathcal R = \{R_1=p(x,y) \rightarrow r(x,y),$ \mbox{$R_2=r(x,y) \rightarrow \exists z.~ q(x,z),$} \mbox{$R_3=r(x,y) \rightarrow t(x)\}$.} Let $\mathcal{D}$ be the (terminating) breadth-first derivation of depth $2$ from $F$ whose sequence of associated triggers is $(R_1,\pi), (R_3,\pi)$ with \mbox{$\pi=\{x\mapsto a, y\mapsto b\}$}, which produces 
$r(a,b), t(a)$;
note that the trigger $(R_2,\pi)$ is \SO-applicable but not \RE-applicable because of the presence of $q(a,c)$ in $F$. Let  $F' = \{p(a,b)\}$. The restriction of $\mathcal{D}$ induced by $F'$ is a restricted chase derivation of depth $2$, however, it is not breadth-first since now $(R_2,\pi)$ is \mbox{\RE-applicable} at rank $2$ and thus has to be applied (for the derivation to be breadth-first). 
\end{example}

This motivates the definition of  preservation of ancestry. An X-chase preserves ancestry if\corvtwo{,} for any atom $A$ in an X-derivation, there is an X-derivation that starts from the prime ancestors of $A$ and is able to produce $A$ at the same rank. 

\begin{definition}[Preservation of Ancestry]
The X-chase \emph{preserves ancestry} if, for every X-derivation $\mathcal{D}$ from $(F,\mathcal R)$ and 
 atom \mbox{$A\in F\DD \setminus F$,}
  there is an X-derivation $\mathcal{D}'$ from $(Anc^0_{\mathcal{D}}(A),\mathcal{R})$ such that $A\in F^{\mathcal D'}$ and \mbox{$rank\dD(A)=rank_{\mathcal D'}(A)$.}
\end{definition}
 


We will show that $k$-boundedness is decidable for any chase variant that satisfies this property. 
\correction{We achieve this by limiting the size of the factbases we need to consider. This implies that a finite number of factbases suffice in order to test $k$-boundedness of a given ruleset. To formally establish this implication, we introduce the following notion\corvtwo{.} Let $F$ and $F'$ be atomsets and \mbox{$\tau:\cnst(F)\rightarrow\cnst(F')$,} \mbox{$\sigma:\var(F)\rightarrow\var(F')$} be mappings of constants and variables respectively. If $h=\tau\cup\sigma$ 
is a bijection \corvtwo{such} that $h(F)= F'$\corvtwo{,} then $h$ is a \emph{\mbox{quasi-isomorphism}} from $F$ to $F'$. Two knowledge bases with the same ruleset and quasi-isomorphic factbases behave equivalently with respect to any chase variant.}

\begin{theorem}\label{presantheorem}{Determining if a set of rules is X-$k$-bounded is decidable if the X-chase preserves ancestry. }
\end{theorem}
\begin{proof} Let X be a chase variant that preserves ancestry. Let $\mathcal R$ be a ruleset. Suppose that $\mathcal R$ is not X-$k$-bounded. Therefore, there is a factbase $F$ and a derivation $\D$ from $(F,\mathcal R)$ with depth strictly greater than $k$. So there exists an atom $A\in F\DD$ with $rank(A)=k+1$. Because the X-chase preserves ancestry, there exists an X-derivation $\D'$ from $(Anc^0_{\mathcal D}(A),\mathcal R)$ which produces $A$ with the same rank as $\D$. Therefore, $\D'$ is also of depth more than $k$. Let $b$ be the maximum number of atoms in the bodies of the rules of $\mathcal R$. By Proposition~\ref{anclue}, $Anc^0_{\mathcal D}$ has at most $b^{k+1}$ atoms. We have shown that if a ruleset $\mathcal R$ is not X-$k$-bounded, then there exists a factbase $F'$ of at most $b^{k+1}$ (where $b$ depends on $\mathcal R$) such that there is an X-derivation from $(F',\mathcal R)$ of depth strictly greater than $k$. The converse of this statement is trivially true. In conclusion, if X is a chase variant that preserves ancestry, then a ruleset $\mathcal R$ (with $b$ maximum body size) is X-$k$-bounded if and only if for every factbase $F'$ of size at most $b^{k+1}$, every \mbox{X-derivation} from $(F',\mathcal R)$ is of depth at most $k$. Up to quasi-isomorphism, there is a finite number of factbases \corvtwo{of cardinality less or equal to} $b^{k+1}$ and for a given factbase $F$, there is a finite number of X-derivations from $(F, \mathcal R)$. Hence we can indeed compute all these derivations and verify whether $\mathcal R$ is X-$k$-bounded or not.
\end{proof}

%
%
%
%
%
%
%
%
%
%
%

\medskip
We now show that the breadth-first variants of the semi-oblivious and restricted chases preserve ancestry. 
\correction{As illustrated by Examples \ref{bf-so-ancestry-ex} and \ref{bf-r-ancestry-ex}, the restriction of a derivation to the prime ancestors of an atom $A$ 
may not \corvtwo{satisfy} the conditions required by ancestry preservation, although it preserves the rank of $A$, because it may not be a breadth-first derivation. However, we will see that if we take only the subsequence of triggers that produce ancestors of a given atom $A$, 
\corvtwo{then} we can complete it by missing triggers to obtain a breadth-first derivation which preserves the rank of $A$. 
For other chase variants that preserve ancestry as well, we will prove that they satisfy a stronger property, namely heredity.}

First, notice that, when reducing the factbase, the application of the same triggers does not necessarily preserve their ranks and hence the ranks of the atoms they produce. Here is a simple example:
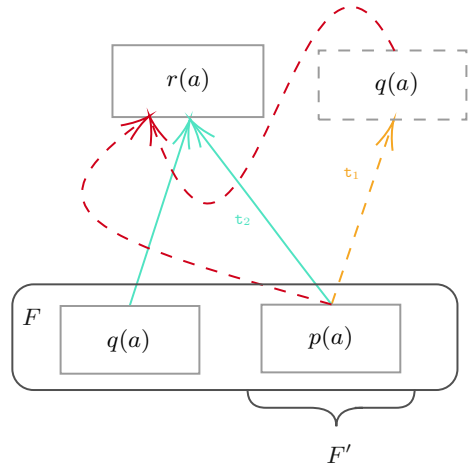
\begin{figure}[b!]
\centering

\tikzset{every picture/.style={line width=0.75pt}} 

\begin{tikzpicture}[x=1pt,y=1pt,yscale=-1,xscale=1]
\clip(85,160) rectangle (300,337);
\draw [color={rgb, 255:red, 80; green, 227; blue, 194 }  ,draw opacity=1 ]   (155,274) -- (177.22,205.22) ;
\draw [shift={(177.83,203.32)}, rotate = 467.9] [color={rgb, 255:red, 80; green, 227; blue, 194 }  ,draw opacity=1 ][line width=0.75]    (10.93,-3.29) .. controls (6.95,-1.4) and (3.31,-0.3) .. (0,0) .. controls (3.31,0.3) and (6.95,1.4) .. (10.93,3.29)   ;

\draw [color={rgb, 255:red, 245; green, 166; blue, 35 }  ,draw opacity=1 ] [dash pattern={on 4.5pt off 4.5pt}]  (231.38,273.43) -- (254.19,205.71) ;
\draw [shift={(254.83,203.82)}, rotate = 468.62] [color={rgb, 255:red, 245; green, 166; blue, 35 }  ,draw opacity=1 ][line width=0.75]    (10.93,-3.29) .. controls (6.95,-1.4) and (3.31,-0.3) .. (0,0) .. controls (3.31,0.3) and (6.95,1.4) .. (10.93,3.29)   ;

\draw  [color={rgb, 255:red, 155; green, 155; blue, 155 }  ,draw opacity=1 ] (128.88,274) -- (181.38,274) -- (181.38,299.63) -- (128.88,299.63) -- cycle ;
\draw  [color={rgb, 255:red, 155; green, 155; blue, 155 }  ,draw opacity=1 ] (204.88,273.5) -- (257.38,273.5) -- (257.38,299.13) -- (204.88,299.13) -- cycle ;
\draw  [color={rgb, 255:red, 155; green, 155; blue, 155 }  ,draw opacity=1 ] (148.5,175.28) -- (205.6,175.28) -- (205.6,202.48) -- (148.5,202.48) -- cycle ;
\draw  [color={rgb, 255:red, 155; green, 155; blue, 155 }  ,draw opacity=1 ][dash pattern={on 4.5pt off 4.5pt}] (226.5,177.28) -- (282.38,177.28) -- (282.38,203.43) -- (226.5,203.43) -- cycle ;
\draw [color={rgb, 255:red, 80; green, 227; blue, 194 }  ,draw opacity=1 ]   (231.38,273.43) .. controls (211.4,247.79) and (215.91,253.02) .. (178.97,204.79) ;
\draw [shift={(177.83,203.32)}, rotate = 412.53999999999996] [color={rgb, 255:red, 80; green, 227; blue, 194 }  ,draw opacity=1 ][line width=0.75]    (10.93,-3.29) .. controls (6.95,-1.4) and (3.31,-0.3) .. (0,0) .. controls (3.31,0.3) and (6.95,1.4) .. (10.93,3.29)   ;

\draw [color={rgb, 255:red, 208; green, 2; blue, 27 }  ,draw opacity=1 ] [dash pattern={on 4.5pt off 4.5pt}]  (255.17,177.41) .. controls (218.02,102.45) and (200.01,311.3) .. (163.06,203.72) ;
\draw [shift={(162.5,202.08)}, rotate = 431.46000000000004] [color={rgb, 255:red, 208; green, 2; blue, 27 }  ,draw opacity=1 ][line width=0.75]    (10.93,-3.29) .. controls (6.95,-1.4) and (3.31,-0.3) .. (0,0) .. controls (3.31,0.3) and (6.95,1.4) .. (10.93,3.29)   ;

\draw [color={rgb, 255:red, 208; green, 2; blue, 27 }  ,draw opacity=1 ] [dash pattern={on 4.5pt off 4.5pt}]  (231.38,273.43) .. controls (144.05,248.99) and (109.85,239.6) .. (160.93,203.19) ;
\draw [shift={(162.5,202.08)}, rotate = 505.01] [color={rgb, 255:red, 208; green, 2; blue, 27 }  ,draw opacity=1 ][line width=0.75]    (10.93,-3.29) .. controls (6.95,-1.4) and (3.31,-0.3) .. (0,0) .. controls (3.31,0.3) and (6.95,1.4) .. (10.93,3.29)   ;

\draw  [color={rgb, 255:red, 80; green, 80; blue, 80 }  ,draw opacity=1 ] (199.6,305.53) .. controls (199.6,310.2) and (201.93,312.53) .. (206.6,312.53) -- (223.19,312.53) .. controls (229.86,312.53) and (233.19,314.86) .. (233.19,319.53) .. controls (233.19,314.86) and (236.52,312.53) .. (243.19,312.53)(240.19,312.53) -- (255.8,312.53) .. controls (260.47,312.53) and (262.8,310.2) .. (262.8,305.53) ;
\draw  [color={rgb, 255:red, 80; green, 80; blue, 80 }  ,draw opacity=1 ] (110.63,273.81) .. controls (110.63,269.39) and (114.21,265.81) .. (118.63,265.81) -- (271,265.81) .. controls (275.42,265.81) and (279,269.39) .. (279,273.81) -- (279,297.81) .. controls (279,302.23) and (275.42,305.81) .. (271,305.81) -- (118.63,305.81) .. controls (114.21,305.81) and (110.63,302.23) .. (110.63,297.81) -- cycle ;

\draw (234.37,329.97) node [color={rgb, 255:red, 19; green, 20; blue, 19 }  ,opacity=1 ]  {$F'$};
\draw (244.5,151.4) node [scale=0.7,color={rgb, 255:red, 208; green, 2; blue, 27 }  ,opacity=1 ]  {$\mathtt{t}_{2}$};
\draw (231.13,286.31) node   {$p( a)$};
\draw (154.67,287.33) node   {$q( a)$};
\draw (256.06,189.95) node   {$q( a)$};
\draw (177.17,188.5) node   {$r( a)$};
\draw (239.17,223.74) node [scale=0.7,color={rgb, 255:red, 245; green, 166; blue, 35 }  ,opacity=1 ]  {${\textstyle \mathtt{t}_{1}}$};
\draw (197.9,241.38) node [scale=0.7,color={rgb, 255:red, 80; green, 227; blue, 194 }  ,opacity=1 ]  {$\mathtt{t}_{2}$};
\draw (117.93,278.27) node [color={rgb, 255:red, 20; green, 20; blue, 20 }  ,opacity=1 ]  {$F$};

\end{tikzpicture}

\caption{Chase graph(s) associated with the derivations $\D$ and $\D'$ of Example~\ref{exampledifferentranks} (dashed elements do not appear in the chase graph of $\D$).}\label{figurexdifr}
\end{figure}
\noindent
\begin{example}\label{exampledifferentranks} Let $\mathcal R$ contain the rules $R_1: p(x)\rightarrow q(x)$ and $R_2: p(x)\wedge q(x)\rightarrow r(x)$. Let $F=\{p(a),q(a)\}$ and we denote $\t_1=(R_1,\{x\mapsto a\})$ and $\t_2=(R_2,\{x\mapsto a\})$. Here is a derivation $\D$ from $(F,\mathcal R)$:
\begin{center}
$(\emptyset,F),(\t_1,F),(\t_2,F\cup\{r(a)\})$
\end{center}
Now\corvtwo{, for $F' = \{p(a)\}$,} there is a derivation $\D'$ from $(F',\mathcal R)$ such that \mbox{$\trg(\D')=\trg(\D)$ :}
\begin{center}
$(\emptyset,F'),(\t_1,F),(\t_2,F\cup\{r(a)\})$
\end{center}
Notice that the trigger $\t_2$ and correspondingly the atom $r(a)$ that it produces, have different ranks in the two derivations. In Figure~\ref{figurexdifr} we have the corresponding chase graph(s).
\end{example}

\bigskip
{However, it follows from Proposition \ref{ranklowerbound} that
 by taking the restriction of a \bfX-derivation the ranks of atoms can only increase. Indeed, the restriction of a \bfX-derivation from $(F,\mathcal R)$ is always an $\OB$-derivation from $(F,\mathcal R)$. }
We conclude that  although the ranks can be different, those obtained by the larger factbase provide a lower \corvtwo{bound}.

The following proposition applies to any breadth-first oblivious, semi-oblivious or restricted derivation: it shows that when we restrict the factbase to the prime ancestors of any atom $A$ produced in a derivation $\D$, the \bfO, \bfSO~and \bfR ~ranks of all the ancestors of $A$ in the smaller factbase are equal to their ranks in $\D$.

\begin{proposition}[Preservation of Ranks of Ancestors]\label{rAmks}
Let \D be a \bfX-derivation from $(F,\mathcal R)$, where $\text X\in\{\OB,\SO, \RE\}$. Let $A\in F\DD$ be an atom and $\D'$ any \bfO-derivation from $(Anc^0_{\mathcal D}(A),\mathcal R)$ of the same depth as $\D$.
Then\corvtwo{, \mbox{$rank\dD(A')=rank_{\mathcal D'}(A')$} for every \mbox{$A'\in Anc\dD(A)$}.}
\end{proposition}
\begin{proof} Let $\D_A$ be the \OB-derivation from $(Anc^0_{\mathcal D}(A),\mathcal R)$ given by 
the maximal subsequence of $\trg(\D)$ made only by 
triggers that produce an ancestor of $A$.
\corvtwo{We first show that $rank\dD(A')=rank_{\mathcal D_A}(A')$, for every $A'\in Anc\dD(A)$.}
We use induction on the number of triggers in $\D_A$.
If $\D_A$ has no trigger\corvtwo{,} then $A\in F$ as the atom has no ancestors, and the claim follows. Now, assume that $\D_A$ has $i$ triggers and let $\t_i$ be its last trigger.
All the atoms that are in $\spp(\t_i)$ have been produced by the first $i-1$ triggers of $\D_A$.
By \corvtwo{the} inductive hypothesis, since such atoms are ancestors of $A$, they
have the same rank in $\D$ and $\D_A$.
Hence, all \corvtwo{ancestors of $A$ in $\op(\t_i)$} have the same rank in $\D$ and $\D_A$.

Now let $\D'$ be any \bfO-derivation from $(Anc^0_{\mathcal D}(A),\mathcal R)$.
This can be seen as an \OB-derivation from $(F,\mathcal R)$. So, by Proposition \ref{ranklowerbound},
the ranks of atoms in $\D'$ can only increase with respect to $\D$.
 But again by Proposition~\ref{ranklowerbound}, the ranks of atoms in $\D_A$ can only increase with respect to $\D'$. Since the ranks of the ancestors of $A$ are the same in $\D_A$ and $\D$, we conclude that so is the case for $\D'$.
%
\end{proof}

\medskip

We are now ready to prove that the breadth-first semi-oblivious and restricted chase variants preserve ancestry.

\begin{theorem}\label{ch1}{The \bfX-chase preserves ancestry when \mbox{$\text{X}\in\{\SO,\RE \}$.}}
\end{theorem}

\begin{proof} We assume that $\mathcal{D}$ is an X-derivation from $(F,\mathcal{R})$ and $\t$ is a trigger  that produces an atom $A$  of rank $n > 0$ in $\D$. 

We start from the derivation $\D_{\mid Anc_{\mathcal D}^0(A)}$  and remove all triggers that do not produce any ancestor of $A$ in $\D$. Let $\D''$ be the obtained derivation.
We build a derivation $\D'$ by applying the following completion procedure on $\D''$.
For each rank $1\leq k\leq n$\corvtwo{,} we first extend $\D'$ by applying all triggers of rank $k$ in $\D''$, in the given order.
Then, within the same rank, we further extend the resulting derivation by non-deterministically applying any sequence of X-applicable triggers of rank $k$ until no trigger of rank $k$ is X-applicable anymore. We repeat this process rank by rank.

\smallskip
We claim that $\D'$ is a \bfX-derivation.
The first thing to show is that $\D'$ is rank-compatible.
To do so, first notice that completion yields a derivation from $(Anc_{\mathcal D}^0(A),\mathcal R)$ which, by Proposition \ref{ranklowerbound}, \shmeiwsh{produces a subset of the atoms produced by a $\bfO$-derivation from $(Anc_{\mathcal D}^0(A),\mathcal R)$ at the same rank.} Then, by Proposition \ref{rAmks}, the ranks of the ancestors of $A$ are the same on $\D$ and the $\bfO$ derivation, thus \mbox{$rank_{\mathcal D}(\t)= rank_{\mathcal D'}(\t)$} for all $\t\in\trg(\D'')$. We conclude that $\D'$ is rank-compatible, as the completion does not alter the ranks of the triggers.
By construction, every rank in $\D'$ cannot be extended with other \mbox{X-applicable} triggers. 
So, \shmeiwsh{if $\D'$ is an X-derivation, then it respects both conditions for being a breadth-first X-derivation. 
What remains to be shown is that all triggers in $\D'$ (and specifically those of $\D''$) are X-applicable on their respective prefix}. 
We perform a case analysis on the type of chase.

\medskip
\noindent (Case $\text{X}$=$\SO$)

We say that $(R,\pi)$ and $(R,\pi')$ are \SO-equivalent if \corvtwo{$\pi(x)=\pi'(x)$ for all $x\in\fr(R)$}.


\smallskip
We start from the breadth\corvtwo{-}first completion $\D'$ that has been described above.
Let $\t$ be the first trigger of $\D'$ at position $ i +1$ in $\trg(\D')$ that is not $\SO$-applicable on $\D'_{\mid i }$. 
Then, there exists a trigger $\t'$ in $\D'_{\mid i}$ which is $\SO$-equivalent with $\t$. \corvtwo{Notice that, by construction of $\D'$, all elements of $\trg(\D')\setminus\trg(\D)$ are \SO-applicable on their respective prefixes, so $\t\in\trg(\D)$.
On the contrary, $\t'$ cannot occur in $\D$ because $\t$ occurs in $\D$ and no pair of distinct triggers in $\D$ can be $\SO$-equivalent.
Therefore, $\t'$ has been introduced by the completion procedure. According to this procedure, all triggers of $\D'$ that precede $\t$ within the same rank as $\t$ also occur in $\D$.
Hence, the rank of $\t'$ in $\D'$ is strictly lower than that of $\t$. Now, because $\D'_{\mid i}$ is a \mbox{\bfSO-derivation,} we can apply Propositions~\ref{rAmks} and~\ref{bfminrank} to conclude that \mbox{$rank_{\mathcal D}(\t)= rank_{\mathcal D'}(\t)$,} so also \mbox{$rank_{\mathcal D'}(\t')<rank_{\mathcal D}(\t)$.}}

At this point we do not know whether $\t'$ was applicable on some prefix of $\D$. This would be the case if $\spp(\t')$ was produced by $\D$. However, $\spp(\t')$ \corvtwo{may} contain atoms that use some fresh terms introduced by a previous trigger added by the completion. 
 \corvtwo{Nevertheless, using Proposition~\ref{retractionthe} we can find a trigger that is \mbox{\SO-equivalent} with $\t'$ (so also with $\t$) and has to be applicable on some prefix of $\D$, leading to a contradiction. 
In particular, this proposition 
implies that there is a retraction $h$ from $F^{\mathcal D'_{\mid i}}\cup F\DD$ to $F\DD$ that maps $\spp(\t')$ to atoms of equal or lower rank in $\D$. Suppose that \mbox{$\t=(R,\pi)$} and $\t'=(R,\pi')$. The \mbox{SO-equivalence} of $t$ and $t'$ guarantees that $\pi$ and $\pi'$ agree on the mapping of all frontier variables of $R$. Those variables are necessarily mapped by $\pi$ (so also by $\pi'$) to terms of $F\DD$.
Let \mbox{$\t''=(R,h\circ\pi')$.} We have that \mbox{$\spp(\t'')=h(\spp(\t'))$} so $\t''$ is applicable on $F^{\mathcal D}$, with \mbox{$rank\dD(\t'')\leq rank_{\mathcal D'}(\t')$} (again, the latter is a consequence of Proposition~\ref{retractionthe}).}

\corvtwo{Since $h$ does not affect any variable of $F\DD$, it does not affect the mapping of the frontier variables of $R$, i.e., for every $x\in\fr(R)$, \mbox{$h\circ\pi'(x)=\pi'(x)$.} This implies that $\t''$ is \mbox{\SO-equivalent} with $\t'$, so also with $\t$. 
Let $\D''$ be the prefix of $\D$ with all elements of rank \shmeiwsh{strictly}  less than $rank_{\mathcal D}(\t)$. $\D''$ does not include $\t$ nor any trigger of the same \SO-equivalence class. 
But $\spp(\t'')=h(\spp(\t'))\subseteq F^{\mathcal D}$ and, from $rank\dD(\t'') < rank\dD(\t)$, we know in particular that $\spp(\t'')\subseteq F^{\mathcal D''}$. So $\t''$ is \SO-applicable on $\D''$. But we know that $\t''\not\in\trg(\D)$ because it is \SO-equivalent with $\t$. This is a contradiction because $\D$ is breadth-first, so $\t''$ must have been applied at its respective rank. We conclude that all the triggers of $\D'$ are \mbox{\SO-applicable} on their respective prefix and $\D'$ is a $\bfSO$-derivation.
}

\medskip
\noindent (Case $\text{X}$=$\RE$)

Let $\t_{ i +1}$ be the first trigger of $\D'$ at position $ i +1$ in $\trg(\D')$ that is not $\RE$-applicable on $\D'_{\mid i }$. 
Assume \corvtwo{that} $\t_{ i +1}$ is of rank $m+1$ in $\D'$
and let ${\mathcal D}_{\mid depth(m)}'$ be the maximal prefix of $\D'$ of depth $m$.
Hence, $F^{{\mathcal D_{\mid  i }'}}=F^{{\mathcal D}_{\mid depth(m)}'}\cup\op(\t_{n+1})\cup\cdots\cup\op(\t_{ i })$  where, by construction of~$\D'$, the (possibly empty) sequence of triggers $\t_{n+1},...,\t_{ i }$  of $\D'$  have the same rank as $\t_{ i +1}$.


\smallskip
Now, let $\D_{\mid depth(m)}$ be the maximal prefix of $\D$ of depth $m$. Of course $\D_{\mid depth(m)}$ is a breadth-first \RE-derivation with the same depth as ${\mathcal D}_{\mid depth(m)}'$. 
Let $\D''$ be the derivation from $(F,\mathcal R)$ with $\trg(\D'')=\trg(\D')$. Then $\D''$ produces the same atoms as $\D'$ and does not have greater depth. 
So, by Proposition~\ref{retractionthe}, there is a retraction $h$ from \mbox{$F^{{\mathcal D}_{\mid depth(m)}''}\cup F^{{\mathcal D}_{\mid depth(m)}}$} to \mbox{$F^{{\mathcal D}_{\mid depth(m)}}$,} so also from \mbox{$F^{{\mathcal D}_{\mid depth(m)}'}\cup F^{{\mathcal D}_{\mid depth(m)}}$} to \mbox{$F^{{\mathcal D}_{\mid depth(m)}}$.} 


\smallskip
Recall that 
$\t_{n+1},...,\t_{ i +1}$
are  the first triggers  at the beginning of rank $m+1$ in $\D'$.
{These triggers are necessarily producing an ancestor of $A$ since it is the case for $\t_{ i +1}$. And since $\D'$ is rank-compatible, they also 
have rank $m+1$ in $\D$. } 
By Proposition \ref{rAmks}, the ranks of the ancestors of $A$ are the same in $\D$ and $\D'$.
So the support of each trigger is in $F^{{\mathcal D}_{\mid depth(m)}}$ and since 
$i)$ $h$ preserves the terms of  $F^{{\mathcal D}_{\mid depth(m)}}$
and
$ii)$ fresh nulls are named after the trigger that generated them,
we deduce that $h$ behaves as the identity on
$\op(\t_{n+1})\cup\cdots\cup\op(\t_{{ i +1}})$.
We get
 $h(F^{{\mathcal D}_{\mid i }'})=h(F^{{\mathcal D}_{\mid depth(m)}'}\cup\op(\t_{n+1})\cup\cdots\cup\op(\t_{ i }))\subseteq  F^{{{\mathcal D}_{\mid depth(m)}}}\cup \op(\t_{n+1})\cup\cdots\cup\op(\t_{{ i }})\subseteq F^{\mathcal D_{\mid j}}$
where $j+1$ is the position of $\t_{ i +1} $ in $\trg(\D)$.

\smallskip

Now, as $\t_{ i +1}$ is not \mbox{$\mathbf{R}$-applicable} on $\D_ i '$ there is a retraction $\sigma$ from 
$F^{{\mathcal D}_{\mid  i+1 }'}
$ to $F^{{\mathcal D}_{\mid  i }'}$.
But this means that $\t_{ i +1}$ was not applicable on the prefix $\D_{\mid j}$  of $\D$
because of the retraction $h\circ\sigma$ from 
$F^{{\mathcal D}_{\mid  j+1 }}
$ to $F^{{\mathcal D}_{\mid  j }}$; 
indeed, $h\circ\sigma(\op(\t_{ i +1}))\subseteq h(F^{{\mathcal D}_{\mid i }'})\subseteq F^{\mathcal D_{\mid j}}$.  \corvtwo{This is} a contradiction.
\end{proof}


\subsection{Heredity}

Heredity is a second property that leads to the decidability of $k$-boundedness.
A chase  variant X is hereditary if by restricting an X-derivation to a subset of a factbase we still obtain an \mbox{X-derivation.}


\begin{definition}[Heredity] The X-chase is \emph{hereditary} if\corvtwo{, for any X-derivation \D from $(F,\mathcal R)$  and subset $G\subseteq F$, the restriction $\D_{\mid G}$ is also an X-derivation.} 
\end{definition}

This property is satisfied by the oblivious, the semi-oblivious and the restricted chase variants.

\begin{theorem}\label{hered}{The $\mathrm{X}$-chase is hereditary for $\text{X}\in\{\mathbf{O},\bfO,\mathbf{SO},\mathbf{R}\}$.}  
\end{theorem}

\begin{proof} 
Let $\mathcal{D}$ be an X-derivation from $(F,\mathcal{R})$
and $G\subseteq F$. We do a case analysis on $X$.

\noindent  (Case $\text{X}$=$\OB$) Clearly $\D_{\mid G}$ is an  $\mathbf{O}$-chase  derivation, therefore the \OB-chase is hereditary.

\noindent  (Case $\text{X}$=$\bfO$) Since $\mathcal{D}$ is rank-compatible and since the ordering of triggers is preserved in $\mathcal{D}_{\mid G}$,  we get that $\mathcal{D}_{\mid G}$ is rank-compatible. Moreover, because $\mathcal{D}$ is a \bfO-derivation, all triggers which are 
applicable on atoms whose prime ancestors are in $G$ are also applicable on $\mathcal{D}_{\mid G}$. Therefore $\mathcal{D}_{\mid G}$ is also breadth-first, since at every rank, all possible rule applications are performed.

\noindent  (Case $\mathrm{X}$=$\SO$) The condition for $\mathbf{SO}$-applicability is that we do not have two triggers from the same rule mapping frontier variables in the same way. 
We know that $\trg{(\D)}$ satisfies this condition, hence so does its subseqence $\trg({\D_{\mid G}})$.

\noindent (Case $\text{X}$=$\RE$) Let $\D'$ be a prefix of $\D$. The condition for $\mathbf{R}$-applicability of a trigger $\t$ on $\D'$ is that there is no retraction from the immediate derivation from $F^{\mathcal D'}$ through $\t$ back to $F^{\mathcal D'}$. Since $\mathcal{D}_{\mid G}$ generates 
factbases that are included in the factbases generated by $\mathcal{D}$ at the moment of the application of the same trigger $\t$, we conclude that $\mathbf{R}$-applicability is preserved. 
\end{proof}

Preservation of ancestry is a generalization of heredity, as we state in Theorem \ref{prophereditary}. Moreover, this generalization is strict, since the breadth-first semi-oblivious and breadth-first restricted variants preserve ancestry but are not hereditary, as shown in Examples \ref{bf-so-ancestry-ex} and \ref{bf-r-ancestry-ex}. 

\begin{theorem}
\label{prophereditary}
Every hereditary chase variant preserves ancestry.
\end{theorem}

\begin{proof}
Let X be a hereditary chase variant. 
Let $\D$ be an X-derivation from $(F,\mathcal R)$. 
Let \mbox{$G=Anc^0_{\mathcal D}(A)$.} 
Since X is hereditary, $\D_{\mid G}$ is an X-derivation.
By Proposition~\ref{obliviancest}, $\D_{\mid G}$ produces at least $Anc^0_{\mathcal D}\cup \{A\}$. 
To conclude, we must show that $rank_{\mathcal D}(A)=rank_{\mathcal D_{\mid G}}(A)$.

\smallskip

Given $(\t_i,F_i)$ an element of a derivation $\bar\D$ such that $i\geq 1$, we \corvtwo{de}note \mbox{$new_{\bar{\mathcal {D}}}(\t_i)=\op(\t_i)\setminus F_{i-1}$,} the set of atoms produced by $\t_i$ in $\bar{\mathcal {D}}$, i.e., all atoms in the specialization of the rule head used by the trigger that did not already appear in the last factbase.
We \corvtwo{first show that $new_{\mathcal D}(\t_i)\subseteq new_{\mathcal D_{\mid G}}(\t_i)$, for all $\t_i\in\trg(\D_{\mid G})$}. That is, the triggers in the restriction of a derivation potentially produce more facts. 
 Let $\ell\geq i$ be the position of $\t_i$ in $\trg(\D)$.
Then, $G_{i-1}=G\cup\bigcup_{j<i}\op(\t_j)$  
and $F_{\ell-1}\subseteq F\cup\bigcup_{j<i}\op(\t_j)$.
It follows that
\mbox{$new_{\mathcal D}(\t_i)=\op(\t_i)\setminus F_{\ell-1}$
$\subseteq$
$\op(\t_i)\setminus G_{i-1}=new_{\mathcal D_{\mid G}}(\t_i)$.}

\smallskip
We now show that, for all $A_1,A_2\in Anc_{\mathcal D}(A) \cup \{ A \}$, if $A_1$ is a direct ancestor of $A_2$ in $\D$, then $A_1$ is also a direct ancestor of $A_2$ in $\D_{\mid G}$. 
 Let $\t$ be the trigger that produces $A_2$ in \D. We know that $\t\in\trg(\D_{\mid G})$. Since $new_{\mathcal D}(\t)\subseteq new_{{\mathcal D}_{\mid G}}(\t)$ and $A_2\in new_{\mathcal D}(\t)$, the trigger $\t$ produces $A_2$ in ${\mathcal D}_{\mid G}$, i.e., $A_1$ is a direct ancestor of $A_2$ in ${\mathcal D}_{\mid G}$. 
 
%
%
\smallskip
Hence, the subgraphs induced by $Anc_{\mathcal D}(A) \cup \{ A \}$ and  $Anc_{\D_{\mid G}}(A) \cup \{ A \}$ in their respective chase graphs coincide,  and the rank of $A$ in both derivations is the same.
We have shown that there is a derivation  from $(Anc^0_{\mathcal D}(A),\mathcal R)$ that produces $A$ in the same rank as \D, so the X-chase preserves ancestry.
\end{proof}

Gathering the previous results, we can now state that $\forall$-X-$k$-boundedness is decidable for all identified chase variants that preserve ancestry. 

\begin{corollary}
$\forall$-X-$k$-boundedness is decidable when $\text X\in\{\OB,\bfO,\SO,\bfSO,\RE,\bfR\}$.
\end{corollary}

\begin{proof} By Theorem~\ref{ch1}, $\bfSO$ and $\bfR$ preserve ancestry. By Theorem~\ref{hered}, $\mathbf{O}$, $\bfO$, $\mathbf{SO}$ and $\mathbf{R}$ are hereditary, hence they preserve ancestry (by Theorem~\ref{prophereditary}). Finally, by Theorem \ref{presantheorem}, preservation of ancestry is a sufficient condition for the decidability of $\forall$-X-$k$-boundedness.  
\end{proof}

\smallskip
Naturally, this implies the decidability of $\exists$-X-$k$-boundedness when both problems are equivalent, which is in particular the case for the chase variants identified in  Propositions \ref{connection}.

\smallskip
\begin{corollary}
$\exists$-X-$k$-boundedness is decidable when $\text X\in\{\OB,\bfO,\SO,\bfSO\}$.
\end{corollary}

\smallskip
The arguments used to show the decidability of $\forall$-X-$k$-boundedness also allows one to upper-bound the complexity of the problem.

\begin{proposition}
The $\forall$-X-$k$-boundedness problem (where $k$ is unary-encoded) is in 2-EXPTIME for \mbox{$\text X \in\{\bfO,\bfSO\}$} and in  3-EXPTIME for \mbox{$\text X \in \{\OB, \SO,\RE,\bfR\}$.} 
\end{proposition}

\begin{proof}
Clearly, $\bfO$ and $\bfSO$  always produce isomorphic factbases at the same rank. Therefore, to decide whether all $\bfO$-derivations from a given factbase $F$ have depth at most $k$, it is sufficient to build one \bfO-derivation of depth at most $k$ and, if the depth $k$ is reached, check whether one trigger is still $\OB$-applicable, in which case we obtain a counter-example to the \mbox{$\forall$-\bfO-$k$-boundedness.} The same holds for the  \bfSO-chase. On the contrary, the depth of the derivations for the other variants depends on the order in which triggers are applied. Hence, for a given factbase, we have to check all possible orders on triggers, instead of building a single derivation.  

%
%
%
%

Given a ruleset $\mathcal R$ with at most $b$ atoms in the bodies of its rules, the number of factbases of size at most $b^k$ is in the worst case double exponential with respect to $k$. The length of a derivation of depth at most $k$ from a given knowledge base is also double exponential in $k$. Finally the number of all the different derivations of depth at most $k$ is triple exponential in $k$. 
Hence\corvtwo{,} the complexity of the  $\forall$-X-$k$-boundedness problem  (where $k$ is unary-encoded) is in 2-EXPTIME for $\text X \in\{\bfO,\bfSO\}$ and in  3-EXPTIME for $\text X \in \{\OB, \SO,\RE,\bfR\}$. 
\end{proof}


%
%

\subsection{The \EQ-chase and the \bfE-chase do not preserve ancestry}
Finally, the following example shows that the \EQ~and \bfE-chase do not preserve ancestry. 
Hence, the decidability of X-$k$-boundedness for those variants remains an open question. 

\begin{example}
\label{eqchaseanc}
Let $(F,\mathcal R)$ be a knowledge base, where $F=\{p(x_1,x_1), p(x_2,x_2), r(x_1,x_2)\}$ and $\mathcal{R}=\{R\}$, where $R = p(x,x) \wedge p(y,y)\rightarrow p(x,y)\wedge p(y,x)$. Let $\t_1=(R,\{x\mapsto x_1, y\mapsto x_2\})$. Then $\mathcal D=(\emptyset,F),(\t_1,F\cup\{p(x_1,x_2),p(x_2,x_1)\})$ is a (terminating) $\bfE$-derivation from $(F,\mathcal R)$. The set of ancestors of the atom $p(x_1,x_2)$ (of rank $1$ in \D) is $F'=\{p(x_1,x_1), p(x_2,x_2)\}$. However, there is no \EQ-derivation of depth $1$  from $(F',\mathcal R)$ because the obtained factbase would be equivalent to $F'$. Therefore\corvtwo{,} neither the \EQ-chase nor the \bfE-chase preserve ancestry.
\end{example}

\section{Concluding remarks}

\correction
{Let us situate this work in the broader context of ontology-mediated query answering, where the ontology is a set of existential rules.  
The  main techniques investigated to address this issue are either based on the chase or on query rewriting, 
which may both not terminate, since even ground atom entailment from an existential rule base is undecidable (from, e.g., \cite{beeri1981implication}).
However, a wide range of syntactic conditions on rulesets have been defined, which ensure either termination of some chase variant or first-order rewritability for conjunctive queries. A third family of syntactic conditions rely on another decidability paradigm, namely guardedness and its extensions  \cite{cgk08,blm10,kr11,DBLP:conf/ijcai/BagetMRT11,DBLP:conf/kr/ThomazoBMR12}.  }

\correction
{Actually, most sufficient conditions for chase termination apply to the (semi-)oblivious chase: from the simplest ones, 
namely rich-acyclicity \cite{DBLP:conf/pods/HernichS07}, weak-acyclicity \cite{fkmp05} and refinements like joint-acyclicity \cite{kr11} or super-weak-acyclicity \cite{marnette09}, acyclic-GRD \cite{deutsch-nash-remmel08,blms11}, to combinations of acyclicity criteria \cite{DBLP:conf/ecai/BagetGMR14} 
 and model-faithful acyclicity (MFA), which strictly generalizes all the previous acyclicity conditions \cite{chkk13}. 
Recently, MFA was extended to restricted-MFA, which ensures the termination of a specific restricted chase algorithm, even on KBs
 for which the semi-oblivious chase may not terminate \cite{cdk17}. While boundedness  was deeply investigated for Datalog, 
 its study for existential rules is only beginning. Among known classes of rules that ensure chase termination, only acyclic-GRD ensures boundedness
 (of any chase variant), and boundedness cannot be decided for the other classes as they all generalize Datalog. 
 Very recently, work reported in \cite{blmtug19} gave a characterization
of boundedness in terms of chase termination and first-order rewritability. Precisely, a ruleset is X-bounded, for X being the breadth-first oblivious 
 or semi-oblivious chase, if and only if it ensures both X-chase termination and first-order rewritability of conjunctive queries. 
 This characterization allows one to obtain the decidability of (semi-)oblivious-boundedness for classes having decidable chase termination and decidable first-order rewritability,
 such as the important classes of sticky and guarded existential rules.  }

\correction
{
In this article, which extends the work presented in~\cite{corr/abs-1810-09304}, we have followed another path: instead of specific classes of existential rules, we have considered the weaker problem of $k$-boundedness. 
 We have shown that $k$-boundedness is decidable for the main chase variants: oblivious, semi-oblivious and restricted chase, as well as their breadth-first versions. 
 These results rely on establishing a common property that ensures the decidability of $k$-boundedness, namely  ``preservation of ancestry''. 
We note that results concerning the semi-oblivious chase also apply to the logic programs associated with existential rules, 
since the semi-oblivious chase behaves as the Skolem chase. 
}
\corvtwo{We leave for further work the study of the precise complexity of deciding $k$-boundedness according to each kind of chase. }
\correction{Also, the decidability of  \mbox{$\exists$-\RE-k-boundedness} as well as that of \mbox{$\exists$-\bfR-k-boundedness} remain open issues. }
Finally, we leave open the question of the decidability of the $k$-boundedness for chase variants that detect more redundancies,   
such as the equivalent chase and the core chase.

\bigskip

\bibliographystyle{acmtrans}
\bibliography{refs2-CR}



\end{document}